%% file: arxiv-main.tex
\documentclass{article}
\usepackage[utf8]{inputenc}
\usepackage{arxiv-macros}
\usepackage[numbers]{natbib}

\usepackage{geometry}
\usepackage{hyperref}

\renewcommand{\citet}{\cite}
\newif\ifarxiv
\arxivtrue

\title{Differentially Private Non-Convex Optimization under the KL Condition with Optimal Rates}

\usepackage{times}

\newcommand*\samethanks[1][\numexpr\value{footnote}-1]{\footnotemark[#1]}
\author{
\makebox[1.2in]{\hfill Michael Menart \thanks{Department of Computer Science \& Engineering, The Ohio State University,
\href{mailto:menart.2@osu.edu}{menart.2@osu.edu}}
\thanks{Equal Contribution}}
\and
\makebox[1.2in]{\hfill Enayat Ullah \thanks{Department of Computer Science, The Johns Hopkins University,
\href{mailto:enayat@jhu.edu}{enayat@jhu.edu}}
\samethanks}
\and
\makebox[1.2in]{\hfill Raman Arora \thanks{Department of Computer Science, The Johns Hopkins University,
\href{mailto:arora@cs.jhu.edu}{arora@cs.jhu.edu}}}
\and
\makebox[1.0in]{\hfill Raef Bassily\thanks{Department of Computer Science \& Engineering and the Translational Data Analytics Institute (TDAI), The Ohio State University,  \href{mailto:bassily.1@osu.edu}{bassily.1@osu.edu} }}
\makebox[1.8in]{\hfill Crist\'obal Guzm\'an
\thanks{Institute for Mathematical and Computational Engineering, Faculty of Mathematics and School of Engineering, Pontificia Universidad Cat\'olica de Chile,
 \href{mailto:crguzmanp@mat.uc.cl}{crguzmanp@mat.uc.cl} }}}

\usepackage{times}

\usepackage{xcolor}

\newcommand{\rnote}[1]{}
\newcommand{\mnote}[1]{}
\newcommand{\cnote}[1]{}
\newcommand{\enayat}[1]{}
\newcommand{\raman}[1]{}

\date{}

\begin{document}

\maketitle

\begin{abstract}%
We study private empirical risk minimization (ERM) problem for losses satisfying the $(\gamma,\kappa)$-Kurdyka-{\L}ojasiewicz (KL) condition, that is, the empirical loss $F$ satisfies $F(w)-\min_{w}F(w) \leq \gamma^\kappa \|\nabla F(w)\|^\kappa$. The Polyak-{\L}ojasiewicz (PL) condition is a special case of this condition when $\kappa=2$. Specifically, we study this problem under the constraint of $\rho$ zero-concentrated differential privacy (zCDP). When $\kappa\in[1,2]$ and the loss function is Lipschitz and smooth over a sufficiently large region, we provide a new algorithm based on variance reduced gradient descent that achieves the rate $\tilde{O}\big(\big(\frac{\sqrt{d}}{n\sqrt{\rho}}\big)^\kappa\big)$ on the excess empirical risk, where $n$ is the dataset size and $d$ is the dimension. We further show that this rate is nearly optimal. When $\kappa \geq 2$ and the loss is instead Lipschitz and weakly convex, we show it is possible to achieve the rate $\tilde{O}\big(\big(\frac{\sqrt{d}}{n\sqrt{\rho}}\big)^\kappa\big)$ with a private implementation of the proximal point method. When the KL parameters are unknown, we provide a novel modification and analysis of the noisy gradient descent algorithm and show that this algorithm achieves a rate of $\tilde{O}\big(\big(\frac{\sqrt{d}}{n\sqrt{\rho}}\big)^{\frac{2\kappa}{4-\kappa}}\big)$ adaptively, which is nearly optimal when $\kappa = 2$. We further show that, without assuming the KL condition, the same gradient descent algorithm can achieve fast convergence to a stationary point when the gradient stays sufficiently large during the run of the algorithm. Specifically, we show that this algorithm can approximate stationary points of Lipschitz, smooth (and possibly nonconvex) objectives with rate as fast as $\tilde{O}\big(\frac{\sqrt{d}}{n\sqrt{\rho}}\big)$ and never worse than $\tilde{O}\big(\big(\frac{\sqrt{d}}{n\sqrt{\rho}}\big)^{1/2}\big)$. The latter rate matches the best known rate for methods that do not rely on variance reduction. 
\end{abstract}

\input{sections/intro}

\input{sections/prelims}

\input{sections/kappa-leq-2}

\input{sections/kappa-geq-2}

\input{sections/adaptive-gd}

\section*{Acknowledgements}\label{sec:ack}
RB's and MM's research is supported by NSF CAREER Award 2144532, NSF Award AF-1908281, and NSF Award 2112471. 
RA's and EU's research is supported, in part, by NSF BIGDATA award IIS-1838139 and NSF CAREER award IIS-1943251.
CG's research was partially supported by
INRIA Associate Teams project, FONDECYT 1210362 grant, ANID Anillo ACT210005 grant, and National Center for Artificial Intelligence CENIA FB210017, Basal ANID.

\bibliographystyle{alpha} 
\bibliography{refs}

\input{sections/appendix}

\end{document}

%% file: sections/intro.tex
\section{Introduction} \label{sec:intro}
\mnote{This note indicates comments are on} As modern machine learning techniques have increasingly relied on optimizing non-convex objectives, characterizing our ability to solve such problems has become increasingly important. Due to the inherent limitations of solving non-convex optimization problems, that is, the intractability of approximating global minimizers, work in this area has largely focused on approximating stationary points 
\citep{fang2018spider,carmon_convexguilty,Ghadimi:2013,arjevani2019lower,foster2019complexity}, 
or has imposed further restrictions on the loss function. In the latter camp, numerous possible assumptions have been proposed, such as 
the restricted secant inequality \citep{Zhang2013GradientMF} or star/quasar convexity \citep{hinder20a-star-convex}. Perhaps the most promising such condition is the Polyak-{\L}ojasiewicz (PL) condition \citep{polyak63}, and its generalization, the Kurdyka-{\L}ojasiewicz (KL) condition \citep{Kurdyka1998OnGO}\footnote{These conditions are sometimes also referred to as the gradient domination condition. Further, the KL condition is sometimes phrased in terms of $h( F(w)-\min_{w}\{F(w)\})$, for some nondecreasing function $h$, akin to its first appearance \citep{Lezanski62}.
In our work we instead focus on the (commonly studied) case where $h$ is a monomial. %
}. 
A function $F:\re^d\to\re$ satisfies the $(\gamma,\kappa)$-KL condition if for all $w\in\re^d$ it holds that,
\begin{align} \label{eq:kl-condition}
    F(w)-\min_{w}\{F(w)\} \leq \gamma^\kappa\norm{\nabla F(w;S)}^\kappa
\end{align}
That is, the loss lower bounds the gradient norm. The PL condition is the special case where $\kappa=2$.
Both the KL and PL settings have been the subject of numerous works \citep{KN16, foster18, scaman22a-local-smoothness}. The KL condition, in addition to being weaker than many of the previously mentioned conditions, has led to a number of strong convergence rate results. Furthermore, an increasingly rich literature has shown that overparameterized models such as neural networks satisfy the KL condition in a number of scenarios \citep{Bassily2018OnEC, CP18-generalization-local-optima, LB21, scaman22a-local-smoothness}. 

On the other hand, the reliance of modern machine learning techniques on large datasets has caused growing concern over user privacy. Overparameterized models are of particular concern due to their ability to memorize training data \citep{sweeney_2021,carlini2019secret,feldman2020neural,brown2021memorization}. %
In response to this concern, differential privacy (DP) has arisen as the most widely accepted method for ensuring the privacy of individuals present in a dataset. Unfortunately, it has been shown in a variety of settings that differentially private learning has fundamental limitations. As a result, characterizing these limitations has been the subject of numerous recent works.

Non-convex optimization under differential privacy is still not well understood. For example, in regards to the task of approximating stationary points in the DP setting, there are still gaps between existing upper and lower bounds \citep{ABGGMU23}. Furthermore, for the problem of approximating global minimizers of non-convex loss functions under DP, it has been shown the best possible rate is only $O\big(\frac{d}{n\epsilon})$, even if the optimization algorithm is allowed exponential running time \citep{Ganesh2022LangevinDA}. In the PL setting however, it has been shown that rates of $\tilde{O}\big(\frac{d}{n^2\epsilon^2}\big)$ on the excess empirical risk are achievable \citep{wang2017differentially,lowy23}. 
Interestingly, this matches the optimal rate for DP optimization in the (much more restrictive) \textit{strongly convex} setting, and subsequently lower bounds for this setting show the rate is optimal \citep{bassily2014private}. 
However, note that while the PL condition admits a large class of non-convex functions, it is still strong enough to exclude many \textit{convex} functions. In contrast, any convex function satisfies the $(R,1)$-KL condition at least in the radius $R$ ball around the minimizer. This can be obtained directly from the definition of convexity and the Cauchy-Schwarz inequality.
Given the promising (but limited) results for PL functions,
the question arises whether similar results can be obtained for the more general class of objectives satisfying the KL condition, particularly since recent work has shown this generalization allows one to capture common models outside the reach of the PL condition \cite{scaman22a-local-smoothness}. In this work, we answer this question in the affirmative, and show that the KL assumption leads to fast rates under differential privacy, even in the absence of convexity. We further provide algorithms which are adaptive in the KL parameters. These results widen the range of non-convex models we can train effectively under DP.

\subsection{Contributions}
In this work, we develop the first algorithms for differentially private empirical risk minimization (ERM) under the $(\gamma,\kappa)$-KL condition without any convexity assumption. %
We show that for sufficiently smooth functions it is possible to achieve a rate of $\tilde{O}\big(\big(\frac{\sqrt{d}}{n\epsilon}\big)^{\kappa}\big)$ on the excess empirical risk for any $\kappa\in [1,2]$. For $\kappa\geq 2$, we give an algorithm which attains the same rate for the strictly larger class of weakly convex functions. This rate is new for any $\kappa \neq 2$.
We further show this rate is near optimal when $1+\Omega(1) \leq \kappa \leq 2$ by leveraging existing lower bounds for convex functions satisfying the growth condition. 
For $1 \leq \kappa \leq 2$, we obtain our upper bound via a novel variant of the Spider algorithm, first proposed in \cite{fang2018spider}. This method allows us to leverage the reduced sensitivity of privatizing gradient \textit{differences} to add less noise, an observation that has been leveraged in several other works studying the problem of finding stationary points under differential privacy  \citep{ABGGMU23, murata-diff2}. We also leverage a novel round structure (i.e. the number of steps before the gradient estimator is reset) for our private Spider algorithm. Whereas previous works have largely used fixed round lengths, our analysis crucially relies on variable round lengths with adaptive stopping.
In the case where $\kappa \geq 2$, we obtain our upper bound using a differentially private implementation of the approximate proximal point method. %

For both these algorithms, our analysis leverages the fact that the KL condition forces large gradients during the run of the algorithm. We further show that these larger gradient norms allow us to add more noise ``for free'', and thus better control the privacy budget. We use this observation to run Spider with a higher noise level than, for example, one would see without the KL condition \citep{ABGGMU23}.

Leveraging this intuition, we further develop a simple variant of noisy gradient descent that automatically scales the noise proportional to the gradient norm. 
We provide a novel analysis to show this algorithm
achieves the rate $\tilde{O}\big(\big(\frac{\sqrt{d}}{n\epsilon}\big)^{\frac{2\kappa}{4-\kappa}} + \big(\frac{1}{n}\big)^{\kappa/2} \big)$ under the $(\gamma,\kappa)$-KL condition when $\kappa\in[1,2]$. This rate is $\tilde{O}\big(\frac{d}{n^2\epsilon^2}\big)$ when $\kappa=2$ (i.e. nearly optimal) %
and is $\tilde{O}\big(\big(\frac{\sqrt{d}}{n\epsilon}\big)^{2/3}\big)$ in the slowest regime ($\kappa=1$). %
This result is adaptive and requires no prior knowledge of the KL parameters. 
We additionally prove that this same gradient descent algorithm can achieve fast convergence guarantees even when the KL condition does not hold. In this case where no KL assumption is made, we settle for convergence to a stationary point as approximating a global minimizer is intractable, in general. Specifically, we show that when the trajectory of noisy
gradient descent encounters mostly points with large gradient norm, the algorithm finds a point with gradient norm $\tilde{O}\big(\frac{\sqrt{d}}{n\epsilon}\big)$. We further establish that in the worst case, the algorithm finds a point with gradient norm at most $\tilde{O}\big(\big(\frac{\sqrt{d}}{n\epsilon}\big)^{1/2}\big)$, recovering the best known rate for noisy gradient descent in this setting.

\subsection{Related Work}
Differentially private optimization by now has a rich literature spanning over a decade. Much of this attention has been directed at the convex setting \citep{CMS,jain2012differentially,kifer2012private, bassily2014private,talwar2014private, TTZ15a, bassily2019private, feldman2020private, asi2021private, bassily2021non}.  The study of differentially private optimization in the non-convex setting is comparatively newer, but has nonetheless been growing rapidly \citep{wang2017differentially,Ganesh2022LangevinDA,ABGGMU23,GLOT23}. 

Currently, research into DP non-convex optimization under the KL condition specifically has been restricted to the special case of the PL condition. Assuming that the loss is Lipschitz, smooth, and satisfies the PL condition, the works \cite{wang2017differentially,lowy23} obtained the rate of $\tilde{O}\br{\frac{d}{n^2\epsilon^2}}$ on the excess empirical risk. This rate is optimal because of existing lower bounds for the strongly convex setting \citep{bassily2014private}. More recently, \cite{yang22-dp-sgda} studied the (more general) minmax optimization problems under differential privacy when the primal objective is assumed to be PL, although the rates therein are slower.
Alternatively, in the \textit{convex} setting, \cite{ALD21} characterized the optimal rates for DP optimization under an assumption known as the growth condition. When convexity is assumed, the KL condition and the growth condition are equivalent \citep[Theorem 5.2]{bolte2017error}. Further, convex functions satisfying the growth condition are a strict subset of (general) KL functions.

There are also a number of works which have studied optimization under the KL condition without privacy considerations. The early works \cite{polyak63,Lezanski62} were the first to show that for gradient descent, linear convergence rates are possible when the objective is smooth and satisfies the PL condition.
More recently, \cite{Bassily2018OnEC} showed that under an additional assumption known as the interpolation condition, \textit{stochastic} gradient descent also achieves linear convergence. %
The works \cite{LB21,scaman22a-local-smoothness} studied more general variants of the PL/KL conditions called the PL*/KL* conditions respectively. Specifically, these works study convergence when the condition holds only over a subset of $\re^d$.

%% file: sections/prelims.tex
\section{Preliminaries}
\paragraph{Empirical Risk Minimization} Let $\cX$ be a data domain and let $S=\bc{x_1,...,x_n}\in\cX^n$ be a dataset of $n$ points. Let $f:\re^d\times\cX \to \re$ be a loss function and define the empirical risk/loss
as $F(w;S) = \frac{1}{n}\sum_{i=1}^n f(w;x_i)$. We denote the set of global minimizers as $\cW^* = \argmin_w F(w;S)$, which we assume is nonempty. We assume we are given some starting point $w_0\in \re^d$ and define the closest global minimizer to $w_0$ as $w^*$. That is $w^*= \argmin_{w\in\cW^*}\bc{\|w_0-w\|}$.
As $\cW^*$
may be non-convex, multiple such minimizers may exist, but it suffices to select one arbitrarily.  We consider the problem of minimizing the excess empirical risk, defined at a point $w$ as $F(w;S)-F(w^*;S)$. We assume throughout that $f$ is $\lip$-Lipschitz continuous over some ball (to be defined later). 
We will denote the $d$-dimensional ball centered at $w$ of radius $\rad$ as $\ball{w}{\rad}$.

\paragraph{KL* Condition} %
Since assuming the loss satisfies the KL condition over all of $\re^d$ is unrealistic in practice (and indeed impossible if Lipschitzness is assumed), 
several works have proposed the modified KL* condition \citep{scaman22a-local-smoothness,LB21}. The exact definition of this condition varies. We use the following definition.
\begin{definition}
A function $F:\re^d\to\re$ %
satisfies the $(\gamma,\kappa)$-KL* condition on $\cS \subset \re^d$ w.r.t. $w'\in\re^d$ if $\forall w \in \cS$ it holds that 
$\gamma^\kappa\|\nabla F(w) \|^\kappa \geq F(w) - F(w')$.%
\end{definition}
We will take $w'=w^*$ (i.e. the closest global minimizer to $w_0$) unless otherwise stated. 
In this case,
under the KL* condition, one equivalently has $\frac{1}{\gamma}(F(w)-F(w^*))^{1/\kappa} \leq \norm{\nabla F(w)}$.  
Prior work studying the PL*/KL* condition has generally further assumed $F(w^*)=0$, but we will avoid this assumption for the sake of generality \citep{LB21,scaman22a-local-smoothness}. %
We note that the condition is often phrased so that the constant $\gamma$ has no exponent, however 
this definition will ease notation in our analysis; a conversion to the standard definition is straightforward. 
For our algorithms, we will show that it is sufficient 
for the KL* condition to hold in a ball around an initial point $w_0$. %
Our guarantees could alternatively be phrased under the condition that the KL* assumptions holds in a ball around $w^*$, %
(see Remark \ref{rem:min-in-ball}, Appendix \ref{app:intro}).

Relevant to our discussion will also be the notion of the $(\lambda,\tau)$-growth condition, which states that for any $w\in\re^d$,  
it holds that $F(w)-F(w^*) \geq \lambda^\tau \|w-w^*\|^{\tau}$.
When the loss function is also assumed to be convex, the KL and growth conditions are in fact equivalent up to parameterization. See Appendix \ref{app:intro} for more details. 

\paragraph{Loss bound}
We assume throughout that one is given a bound $\fnot \geq 0$ such that $F(w_0;S) - F(w^*;S) \leq \fnot$ for some $w_0\in\re^d$. %
However, as our results will assume the KL condition holds at $w_0$, one always has the worst case bound $\fnot \leq \gamma^\kappa \lip^\kappa$ by the fact that the loss is $\lip$-Lipschitz. %

\paragraph{Differential Privacy (DP):} We consider primarily the notion of zero concentrated differential privacy (zCDP). For the purpose of referencing existing work, we also define approximate DP. 
\begin{definition}[$\rho$-zCDP \citep{Bun-zCDP}]
An algorithm $\cA$ is $\rho$-zCDP if for all datasets $S$ and $S'$ differing in one data point and all $\alpha \in (1,\infty)$, it holds that $D_{\alpha}(\cA(S)||\cA(S')) \leq \rho\alpha$, where $D_{\alpha}$ is the $\alpha$-R\'enyi divergence.
\end{definition}
\begin{definition}[$(\epsilon,\delta)$-DP \citep{dwork2006calibrating}]
An algorithm $\cA$ is $(\varepsilon,\delta)$-differentially private if for all datasets $S$ and $S'$ differing in one data point and all events $\cE$ in the range of the $\cA$, we have, $\mathbb{P}\br{\cA(S)\in \cE} \leq     e^\varepsilon \mathbb{P}\br{\cA(S')\in \cE}  +\delta$.  
\end{definition}
\sloppy
zCDP guarantees imply approximate DP guarantees.
Specifically, we note that for any $\delta>0$ and $\epsilon\leq\sqrt{\log(1/\delta)}$, $(\epsilon,\delta)$-DP guarantees can be obtained from our results by setting $\rho=O\br{\epsilon^2/\log(1/\delta)}$ \cite[Proposition 1.3]{Bun-zCDP}.

\paragraph{Weak Convexity} 
A function $F:\re^d\to\mathbb{R}$
is $\wc$-weakly convex w.r.t.~$\|\cdot\|$ if for all $0\leq\lambda\leq 1$ and $w,v\in\re^d$ one has
$f(\lambda w+(1-\lambda)v) \leq \lambda f(w)+(1-\lambda)f(v)+\frac{\wc\lambda(1-\lambda)}{2}\|w-v\|^2.$ %

%% file: sections/kappa-leq-2.tex
\section{Optimal Algorithm for \texorpdfstring{$1\leq \kappa \leq 2$}{}} 
\label{sec:kappa-leq-2}
\begin{algorithm}[h!]
\caption{KL Spider}
\label{alg:kl-spider}
\begin{algorithmic}[1]
\REQUIRE Dataset $S=\bc{x_1,...,x_n}$, Privacy parameter $\rho > 0$, Failure probability $\beta > 0$, Initial point $w_0\in\re^d$, Loss bound $\fnot \leq (\lip\gamma)^\kappa$,  
KL* parameters $(\gamma,\kappa)$, Lipschitz parameter $\lip$, Smoothness parameter $\smooth$

\STATE $w_{0,0} = w_0$, ~ $\hat{\Phi}_0 = \fnot$

\STATE $c=1+\fnot^{\frac{2-\kappa}{\kappa}}\frac{1}{64\gamma^2 \smooth}$

\STATE $K=(1+64(1/\fnot)^{\frac{2-\kappa}{\kappa}}\gamma^2\smooth)\Big[\log(\fnot) +\kappa\log\Big(\frac{n\sqrt{\rho}}{\gamma\lip\sqrt{d}}\Big)\Big]$, ~$\beta' = \frac{\beta}{K}\br{\frac{\gamma\lip\sqrt{Kd}}{n\sqrt{\rho} \fnot^{1/\kappa}}}^{2-\kappa}$

\STATE $\hat{\sigma} = \frac{\lip\sqrt{K}}{n\sqrt{\rho}}$ 

\FOR{$k=1,\ldots, K$}

\STATE $\hat{\Phi}_k = \max\bc{\frac{1}{c}\hat{\Phi}_{k-1},~ \min\Big\{\Big(\frac{32\gamma\lip\sqrt{K d \log(1/\beta')}}{n\sqrt{\rho}}\Big)^\kappa, \fnot\Big\} }$

\STATE $T_k = (\fnot/\hat{\Phi}_k)^{\frac{2-\kappa}{\kappa}}$,  ~$\sigma_k = \frac{\hat{\Phi}_k^{1/\kappa}\sqrt{T_k K}}{\gamma n\sqrt{\rho}}$

\STATE $\nabla_{k,0} = \frac{1}{n}\sum_{i=1}^n \nabla f(w_{k,0};x_i) + b_{k,0}$  where $b_{k,0} \sim \cN(0,\mathbb{I}_d\hat{\sigma}^2)$ 

\STATE $t = 0$

\WHILE{$t \leq T_k$ and $\|\nabla_{t,k}\| \geq \frac{7}{8\gamma}\hat{\Phi}_k^{1/\kappa}$}

\STATE $\eta_{k,t} = \frac{1}{4\gamma\smooth \|\nabla_{k,t}\|}\hat{\Phi}_k^{1/\kappa}$

\STATE $w_{k,t+1} = w_{k,t} - \eta_{k,t} \nabla_{k,t}$

\STATE $\Delta_{k,t+1} = \frac{1}{n}\sum_{i=1}^n [\nabla f(w_{k,t+1};x_i) - \nabla f(w_{k,t};x_i)] + b_{k,t+1}$, where $b_{k,t+1}\sim \cN(0,\mathbb{I}_d\sigma_k^2)$

\STATE $\nabla_{k,t+1} = \nabla_{k,t} + \Delta_{k,t+1}$

\STATE $t = t+1$

\ENDWHILE

\STATE $w_{k+1,0} = w_{k,t-1}$

\ENDFOR

\STATE Return $\bar{w}=w_{K+1,0}$
\end{algorithmic}
\end{algorithm}

\paragraph{Algorithm Overview}
Algorithm \ref{alg:kl-spider} is roughly an implementation of noisy Spider with some key differences. 
Similar to Spider, the algorithm runs over $K$ rounds. At the start of any round $k$, a noisy minibatch gradient estimate $\nabla_{k,0}$, is computed. Throughout the rest of the round, the gradient is estimated using the change in the gradient between iterates. That is, for some $t\geq0$, $\nabla_{k,t} = \nabla_{k,0} + \sum_{j=1}^t \Delta_{k,j}$, where each $\Delta_{k,j}$ corresponds to an estimate of a gradient difference. After each gradient estimate is obtained, a standard (normalized) gradient descent update step is performed.

In contrast to traditional Spider, at the start of each round $k\in[K]$, a target excess risk threshold, $\hat{\Phi}_k$, is set.
The algorithm then uses this threshold to define an adaptive stopping mechanism for the round. 
The stopping condition is needed for the event where the excess risk of the update iterate falls below $\hat{\Phi}_k$ before the end of the phase. 
If this happens, the loss lower bound (and hence the gradient norm lower bound) will not be strong enough for the subsequent iterate. Consequently, the noise added for privacy could be too large and cause the trajectory of the algorithm to diverge. As such, we check to see if the loss has fallen below the target threshold before performing any update. We do this indirectly by checking the gradient norm and using the KL condition, as bounding the sensitivity of the loss itself (to ensure privacy) is more delicate.
Our implementation also uses varying phase lengths such that the length of the $k$'th phase is roughly $(1/\hat{\Phi}_k)^{(2-\kappa)/\kappa}$ (note the exponent is non-negative since $\kappa \leq 2$). 
Specifically, the phases get longer
as the algorithm progresses.
This is due to the fact that as the excess risk decreases, the lower bound on the gradient norm (induced by the KL condition) becomes weaker, leading to progressively slower convergence.
We have the following guarantee on the Algorithm.
\begin{theorem} \label{thm:kl-spider-convergence}
Let $\gamma>0$, $\kappa\in[1,2]$. 
There exists $\rad=\tilde{O}\big(\frac{\fnot^{1/\kappa}}{\gamma\smooth} + \fnot^{\frac{\kappa-1}{\kappa}}\gamma\big)$ such that
if $f$ is $\lip$-Lipschitz and $\smooth$-smooth over $\ball{w_0}{\rad}$, Algorithm \ref{alg:kl-spider} is $\rho$-zCDP. Further, if 
$F(\cdot;S)$ satisfies the $(\gamma,\kappa)$-KL* condition over $\ball{w_0}{\rad}$, with probability at least $1-\beta$ the output of Algorithm \ref{alg:kl-spider} satisfies
\begin{align*}
  F(\bar{w};S) - F(w^*;S) = O\br{\br{\frac{\gamma\lip\sqrt{d  K \log(1/\beta')}}{n\sqrt{\rho}}}^{\kappa}} = \tilde{O}\Bigg(\Bigg(\frac{\gamma\lip\sqrt{d}\sqrt{1+(1/\fnot)^{\frac{2-\kappa}{\kappa}}\gamma^2\smooth}}{n\sqrt{\rho}}\Bigg)^{\kappa}\Bigg),  
\end{align*}%
where $K,\beta'$ are as defined in Algorithm \ref{alg:kl-spider}, namely
$K=(1+64(1/\fnot)^{\frac{2-\kappa}{\kappa}}\gamma^2\smooth)\Big[\log(\fnot) +\kappa\log\Big(\frac{n\sqrt{\rho}}{\gamma\lip\sqrt{d}}\Big)\Big]$, 
and $\beta' = \frac{\beta}{K}\Big(\frac{\gamma\lip\sqrt{Kd}}{n\sqrt{\rho} \fnot^{1/\kappa}}\Big)^{2-\kappa}$. 
\end{theorem}
Note the result can be further simplified by  %
setting $\fnot = (\lip\gamma)^\kappa$ (which is always possible by the KL condition) which makes $(1/\fnot)^{\frac{2-\kappa}{\kappa}}\gamma^2\smooth = \frac{\gamma^\kappa\smooth}{\lip^{2-\kappa}}$. 
We defer the proof of privacy to Appendix \ref{app:kl-spider-privacy}, as it is a standard application of the privacy guarantees of the Gaussian mechanism and composition. In the following, we focus on proving the convergence guarantee of the algorithm.
\paragraph{Convergence Proof for Algorithm \ref{alg:kl-spider}}
Our ability to assume loss properties hold only over $\ball{w_0}{\rad}$ (rather than $\re^d)$ hinges on bounding the trajectory of the algorithm. We assume for the following lemmas that the conditions of Theorem \ref{thm:kl-spider-convergence} hold.
\begin{lemma} \label{lem:kl-spider-trajectory-bound}
For any $k\in[K]$ and $t\in[T_k]$ corresponding to iterates of Algorithm \ref{alg:kl-spider}, it holds with probability 1 that 
$w_{k,t} \in \ball{w_0}{\rad}$
    for some $\rad = \tilde{O}\big(\frac{\fnot^{1/\kappa}}{\gamma\smooth} + \fnot^{\frac{\kappa-1}{\kappa}}\gamma\big)$.
\end{lemma}
The implication of this result is that the algorithm starts in, and never leaves the KL region around $w_0$. Thus the KL property holds at every iterate of the algorithm. We provide a proof in Appendix~\ref{app:kl-spider-trajectory}. 
Note that any $\smooth$-smooth function is also $\smooth'$-smooth for $\smooth' > \smooth$.  
Thus the $\frac{\fnot^{1/\kappa}}{\gamma\smooth}$ term in the distance bound can be made negligible by running the algorithm with $\smooth \geq \fnot^{(2-\kappa)/2}/\gamma$ (although this may increase the rate depending on $\fnot$ and $\gamma$).

Our utility proof for Algorithm \ref{alg:kl-spider}, will crucially rely on the following lemma which bounds the gradient error at any step in terms of the excess risk target, $\hat{\Phi}_k$.
\begin{lemma} \label{lem:kl-spider-grad-err}
With probability at least $1-\beta$, for every $k\in[K]$ and $t\in[T_k]$ indexing the iterates of the algorithm, one has that 
$\|\nabla_{k,t} - \nabla F(w_{k,t};S)\| \leq \frac{1}{8\gamma}\hat{\Phi}_k^{1/\kappa}$. 
\end{lemma}
\begin{proof}
The gradient estimates are generated by using exact gradients plus Gaussian noise, thus
{\small
\begin{align*}
    &\|\nabla_{k,t} - \nabla F(w_{k,t};S)\|^2 \\
    &= \big\|\nabla F(w_{k,0};S) + b_{k,0} + \sum_{j=1}^{t} \big[ \nabla F(w_{k,j};S) - \nabla F(w_{k,j-1};S) + b_{k,j}\big] - \nabla F(w_{k,t};S)\big\|^2 
    = \Big\|\sum_{j=0}^{t}  b_{k,j}\Big\|^2.
\end{align*}}

We can use Gaussian concentration results, see \cite[Lemma 2]{jin2019short}, to conclude that for any $\tau\geq0$,
$\PP\bs{\|\nabla_{k,t} - \nabla F(w_{k,t};S)\| \geq \tau} \leq 2 \exp\br{-\frac{\tau^2}{2d(\hat{\sigma}^2 + \sum_{j=1}^{T_k} \sigma_k^2)}}.$ 
Thus, under the settings of $\hat{\sigma} = \frac{\lip\sqrt{K}}{n\sqrt{\rho}}$ and $\sigma_k = \frac{\hat{\Phi}_k^{1/\kappa}\sqrt{T_k K}}{\gamma n\sqrt{\rho}}$ and $T_k=(\fnot/\hat{\Phi}_k)^{\frac{2-\kappa}{\kappa}}$, one has that with probability at least $1-\beta'$ that:
\begin{align*}
    \|\nabla_{k,t} - \nabla F(w_{k,t};S)\| &\leq 2\sqrt{d\log(1/\beta')}(\hat{\sigma} + \sqrt{T_k}\sigma_k) \\
    &= \frac{2\lip\sqrt{K d\log(1/\beta')}}{n\sqrt{\rho}} + \frac{2\sqrt{K d\log(1/\beta')}}{\gamma n\sqrt{\rho}}\hat{\Phi}_k^{\frac{\kappa-1}{\kappa}}\fnot^{\frac{2-\kappa}{\kappa}} \\
    &\overset{(i)}{\leq} \frac{2\lip\sqrt{K d\log(1/\beta')}}{n\sqrt{\rho}} + \frac{2\sqrt{K d\log(1/\beta')}}{\gamma n\sqrt{\rho}}\fnot^{1/\kappa} \\
    &\overset{(ii)}{\leq} \frac{4\lip\sqrt{K d\log(1/\beta')}}{n\sqrt{\rho}} 
    \overset{(iii)}{\leq} \frac{1}{8\gamma}\hat{\Phi}_k^{1/\kappa}.
\end{align*}
Above, $(i)$ uses 
$\hat{\Phi}_k \leq \fnot$. %
Step $(ii)$ uses that $\fnot \leq (\gamma\lip)^\kappa$ by the KL condition and Lipschitzness. Step $(iii)$ uses the fact that 
$\hat{\Phi}_k \geq \br{\frac{32\gamma\lip\sqrt{K d\log(1/\beta')}}{n\sqrt{\rho}}}^{\kappa}$.

Finally, we observe that for all $k\in[K]$, 
$\hat{\Phi}_k \geq \br{\frac{32\gamma\lip\sqrt{K d \log(1/\beta')}}{n\sqrt{\rho}}}^\kappa$ and $\frac{2-\kappa}{2} \geq 0$. %
Hence, the total number of iterations of the algorithm is at most  
{\small
\begin{align*}
    \sum_{k=1}^K T_k \leq K\br{\fnot\br{\frac{n\sqrt{\rho}}{32\gamma\lip\sqrt{K d \log(1/\beta')}}}^\kappa }^\frac{2-\kappa}{\kappa} \leq K\br{\frac{n\sqrt{\rho}\fnot^{1/\kappa}}{\gamma\lip\sqrt{Kd}}}^{2-\kappa}.
\end{align*}}
Thus by the definition of $\beta'$, over the run of the algorithm, we have with probability at least $1-\beta$ that every gradient estimate satisfies the desired error bound.
\end{proof}

We can now prove the main theorem. 
\begin{proof}[Proof of Theorem \ref{thm:kl-spider-convergence}]
In the following, we condition on the high probability event that the gradient errors are bounded, as shown in Lemma \ref{lem:kl-spider-grad-err}. Further, recall that by Lemma \ref{lem:kl-spider-trajectory-bound} 
the trajectory $\{w_{k,t}\}_{k\in[K], t\in[T_k]}$ is contained in $\ball{w_0}{\rad}$ with probability 1, and that the $(\gamma,\kappa)$-KL* conditions holds over this set.

We will show that at the end of the the $k$'th phase (i.e. the $k$'th iteration of the outer loop), the excess risk is at most $\hat{\Phi}_k$. 
First, consider the case where at some point during the phase the gradient norm stopping condition is reached. 
In this case, the condition in the while loop ensures 
$\|\nabla_{k,t}\| \leq \frac{7}{8\gamma}\hat{\Phi}_k^{1/\kappa}$. Thus by Lemma \ref{lem:kl-spider-grad-err} and a triangle inequality we have
$\|\nabla F(w_{k,t};S)\| \leq \frac{7}{8\gamma}\hat{\Phi}_k^{1/\kappa} + \frac{1}{8\gamma}\hat{\Phi}_k^{1/\kappa} = \frac{1}{\gamma}\hat{\Phi}_k^{1/\kappa}$.
Then by the KL assumption we have that
$F(w_{k,t};S) - F(w^*;S) \leq \gamma^\kappa\|\nabla F(w_{k,t};S)\|^\kappa \leq \gamma^\kappa(\frac{1}{\gamma}\hat{\Phi}_k^{1/\kappa})^\kappa \leq \hat{\Phi}_k$, as desired.

We thus turn towards analyzing the alternative case, where the final iterate of the phase is $w_{k,T_k}$, using an induction argument. Specifically, under the inductive assumption that $F(w_{k,0};S) - F(w^*;S)\leq \hat{\Phi}_{k-1}$, %
we will show that $F(w_{k,T_k};S) - F(w^*;S)\leq \hat{\Phi}_{k}$. For the base case, we clearly have $F(w_{0,0};S) - F(w^*;S)\leq \hat{\Phi}_0 = \fnot$.
Using smoothness and the setting of $\eta_{k,t}$, we can obtain the following descent inequality,
\begin{align*}
F(w_{k,t};S) - F(w_{k,t+1};S)    
 &\geq \frac{1}{16\gamma\smooth}\|\nabla_{k,t}\|\hat{\Phi}_k^{1/\kappa} - \frac{1}{4\smooth}\|\nabla_{k,t} - \nabla F(w_{k,t};S)\|^2 .
\end{align*}
We leave the derivation of the above inequality to Lemma \ref{lem:kl-spider-step-convergence} in Appendix \ref{app:kl-spider-step-convergence}.
We now can use the fact that updates are only performed when $\|\nabla_{k,t}\|\geq \frac{7}{8\gamma}\hat{\Phi}_{k}^{1/\kappa}$ and the bound on the gradient estimate error derived in Lemma \ref{lem:kl-spider-grad-err} to obtain
\begin{align*}
 F(w_{k,t};S) - F(w_{k,t+1};S)  &\geq \frac{1}{32\gamma^2\smooth}\hat{\Phi}_k^{2/\kappa} - \frac{1}{256\gamma^2\smooth}\hat{\Phi}_k^{2/\kappa} 
 \geq \frac{1}{64\gamma^2\smooth}\hat{\Phi}_k^{2/\kappa}.
\end{align*}
Summing over all $T_k = (\fnot/\hat{\Phi}_k)^{\frac{\kappa-2}{2}}$ iterations yields
\begin{align*}
    F(w_{k,0},S) - F(w_{k,T_k};S) &\geq \frac{1}{64\gamma^2\smooth}T_k\hat{\Phi}_k^{2/\kappa} 
    = \frac{1}{64\gamma^2\smooth}\fnot^{\frac{2-\kappa}{\kappa}}\hat{\Phi}_k. 
\end{align*}
We then have the following manipulation leveraging the inductive hypothesis,
\begin{align*}
F(w_{k,0};S) - F(w^*;S) + F(w^*;S) - F(w_{k,T_k};S) &\geq \frac{1}{64\gamma^2\smooth}\fnot^{\frac{2-\kappa}{\kappa}}\hat{\Phi}_k \\
\hat{\Phi}_{k-1} + F(w^*;S) - F(w_{k,T_k};S) &\geq \frac{1}{64\gamma^2\smooth c}\fnot^{\frac{2-\kappa}{\kappa}}\hat{\Phi}_{k-1} \\
\Big(1-\fnot^{\frac{2-\kappa}{\kappa}}\frac{1}{64\gamma^2\smooth c}\Big)\hat{\Phi}_{k-1} &\geq F(w_{k,T_k};S) - F(w^*;S) \\
\hat{\Phi}_k &\geq F(w_{k,T_k};S) - F(w^*;S).
\end{align*}
The last step follows because $\big(1-\fnot^{\frac{2-\kappa}{\kappa}}\frac{1}{64\gamma^2\smooth c}\big) = \frac{1}{c}$ and $\frac{1}{c}\hat{\Phi}_k \leq \hat{\Phi}_{k-1}$.
We have now shown that final iterate of each phase has excess risk at most $\hat{\Phi}_k$.

Now, all that remains is to show that $\hat{\Phi}_K \leq \Big(\frac{32\gamma\lip\sqrt{K d \log(1/\beta')}}{n\sqrt{\rho}}\Big)^{\kappa}$. 
Noting that
\ifarxiv we have \fi
$\hat{\Phi}_K \leq \max\Big\{\big(\frac{1}{c}\big)^K\fnot,\Big(\frac{32\gamma\lip\sqrt{K d \log(1/\beta')}}{n\sqrt{\rho}}\Big)^\kappa\Big\}$ 
it suffices to show that
$\big(\frac{1}{c}\big)^K\fnot \leq \br{\frac{\gamma\lip\sqrt{d}}{n\sqrt{\rho}}}^{\kappa}$.
The inequality 
$\big(\frac{1}{c}\big)^K\fnot \leq \br{\frac{\gamma\lip\sqrt{d}}{n\sqrt{\rho}}}^{\kappa}$ is equivalent to $\frac{\log(\fnot) + \kappa\log(\frac{n\sqrt{\rho}}{\gamma\lip\sqrt{d}})}{\log(c)} \leq K$.
Using the fact that $\log(1+x) \geq \frac{x}{1+x}$ for $x\geq 0$, we can obtain that $\log(c) = \log(1+1/[64\fnot^{\frac{\kappa-2}{\kappa}}\gamma^2\smooth]) \geq (1+64\fnot^{\frac{\kappa-2}{\kappa}}\gamma^2\smooth)^{-1}$. %
It thus suffices to have
$K \geq (1+64(1/\fnot)^{\frac{2-\kappa}{\kappa}}\gamma^2\smooth)\Big[\log(\fnot) +\kappa\log\Big(\frac{n\sqrt{\rho}}{\gamma\lip\sqrt{d}}\Big)\Big]$, which is satisfied by the algorithm. %
\end{proof}

\subsection{Lower Bound}
We now demonstrate a lower bound showing that our upper bound is  nearly optimal. To do this, we leverage an existing lower bound from \cite{ALD21} for functions exhibiting the growth condition.
In Theorem \ref{thm:lb_growth} in Appendix \ref{app:asi-lb-extension}, we extend their result to smooth functions and give a lower bound of 
$\Omega\big((\tau)^{\frac{-1}{\tau-1}}\br{\frac{L_0\sqrt{d}}{\lambda n\epsilon}}^{\frac{\tau}{\tau-1}}\big)$
on excess empirical risk of $(\epsilon, \delta)$-DP procedures for convex functions satisfying $(\lambda, \tau)$-growth.
Combining this result with the fact that the $(\frac{1}{\gamma}, \frac{\kappa}{\kappa-1})$-growth condition and convexity implies the $(\gamma,\kappa)$-KL condition (Theorem 5.2 (ii) in \cite{bolte2017error}, restated as Lemma \ref{lem:growth-implies-kl}), yields the following lower bound. 
 \begin{corollary} \label{cor:kl-lb}
Let $\rad,\lip,\smooth>0$ and $1 < \kappa \leq 2$ such that $\kappa = 1 + \Omega(1)$. For any $\rho$-zCDP algorithm, $\cA$, there exists a dataset, $S$, point $w_0\in\re^d$, and loss function $f$  such that $f$ is $\lip$-Lipschitz and $\smooth$-smooth over $\ball{w_0}{\rad}$ and $F(\cdot;S)$ is $(\gamma,\kappa)$-KL, for which the output of $\cA$ has expected excess empirical risk
$\tilde{\Omega}\br{\br{\frac{\gamma\lip\sqrt{d}}{n\sqrt{\rho}}}^\kappa}$. 
\end{corollary}
Note the bound is independent of $\rad$ and $\smooth$. More details on how to obtain Corollary \ref{cor:kl-lb} from the result of \cite{ALD21} can be found in Appendix \ref{app:kl-lb}.

%% file: sections/kappa-geq-2.tex
\section{Algorithm for \texorpdfstring{$\kappa \geq 2$}{}} \label{sec:kappa-geq-2}
\begin{algorithm}[h]
\caption{(KL) Proximal Point Method}
\label{alg:kl-ppm}
\begin{algorithmic}[1]
\REQUIRE Dataset $S$, Privacy parameter $\rho$, zCDP Optimizer for SC loss $\cA$, Initial point $w_0\in\re^d$, Initial loss bound, $\fnot \geq 0$, Failure probability $\beta$, Lipschitz parameter $\lip$, Weak convexity $\wc$
\STATE $T=(1+32\fnot^{\frac{\kappa-2}{\kappa}}\gamma^2\wc)\Big[\log(\fnot) + \kappa\log\Big(\frac{n\sqrt{\rho}}{\gamma \lip \sqrt{d}}\Big)\Big]$

\STATE $\beta' = \frac{\beta}{T}$

\FOR{$t=1 \ldots T$} 

\STATE $F_t(w;S) := F(w;S)+\wc\norm{w-w_{t-1}}^2$

\STATE $w_{t} = \cA(F_t, w_{t-1},\frac{\rho}{T},\beta')$
\ENDFOR
\STATE Return $w_T$
\end{algorithmic}
\end{algorithm}

\noindent In this section, we assume the loss $F(\cdot;S)$ is $\wc$-weakly convex and that the empirical loss satisfies the $(\gamma,\kappa)$-KL* condition for $\kappa \geq 2$. 
We avoid making a smoothness assumption in this regime. 
When $\kappa > 2$ and the KL* condition holds in a region with small excess risk,
the loss functions cannot be smooth (unless it is the constant function). To elaborate, one can show that the loss upper bound implied by smoothness and the loss lower bound implied by the KL* condition lead to a contradiction. Instead of smoothness, we consider the (strictly weaker) assumption of weak convexity. 
As convex functions are weakly convex with $\wc=0$,
this setting is a strict relaxation of the loss assumptions considered by \cite{ALD21}.
Despite this, we achieve essentially the same rate as theirs. Moreover, in Theorem \ref{thm:lower-bound-kappa-geq-2} in Appendix \ref{app:kappa-geq-2}, we give a lower bound of $\br{\frac{1}{n\epsilon}}^\kappa$, which establishes that our rate is tight (at least) for $d=1$. Its proof adapts the construction in \citet[Theorem 5]{ALD21} from a lower bound on excess population risk under pure, $(\epsilon,0)$-DP to that on excess empirical risk under approximate, $(\epsilon,\delta)$-DP. The lower bound holds for convex functions satisfying the growth condition and thus satisfying the KL condition, via Lemma \ref{lem:growth-implies-kl}.

Our algorithm in this case is simply a differentially private 
implementation of the approximate proximal point method. This method has been used in prior work for non-KL functions to approximate stationary points \citep{DG19, DD:2019, bassily2021differentially}.
We have the following guarantee for this method.
\begin{theorem} \label{thm:kappa-geq-2-result}
\mnote{Changed Lipschitzness to be over $\re^d$ not $\cB_B$, since trajectory only stays in ball w.h.p.} Let $\gamma>0$,
$\kappa \geq 2$, %
There exists $\rad=\tilde{O}\big(\frac{\lip}{\wc}(1+\frac{\sqrt{T d\log{(n^2\log^2{(1/\beta')}/d\beta')}}}{n\sqrt{\rho}})+\lip\fnot^{\frac{\kappa-2}{\kappa}}\gamma^2\big)$ and a subroutine $\cA$ such that if $f$ is Lipschitz %
then Algorithm \ref{alg:kl-ppm} is $\rho$-zCDP. If $F(\cdot;S)$ also satisfies the $(\gamma,\kappa)$-KL* condition and $\wc$-weak convexity over $\ball{w_0}{\rad}$,  
then with probability at least $1-\beta$ the output of Algorithm \ref{alg:kl-ppm} has excess risk, $F(w_T;S) - F(w^*;S)$, at most
\begin{align*}
O\bigg(\bigg(\frac{\gamma \lip\sqrt{T d\log(n^2\log^2{(1/\beta')}/d\beta')}}{n\sqrt{\rho}}\bigg)^{\kappa}\bigg)  = \tilde{O}\bigg(\bigg(\frac{\gamma \lip\sqrt{d(1+\fnot^{\frac{\kappa-2}{\kappa}}\gamma^2\wc)}}{n\sqrt{\rho}}\bigg)^{\kappa}\bigg).
\end{align*}
where {\small $T=(1+32\fnot^{\frac{\kappa-2}{\kappa}}\gamma^2\wc)\Big[\log(\fnot) + \kappa\log\Big(\frac{n\sqrt{\rho}}{\gamma \lip \sqrt{d}}\Big)\Big]$} 
and $\beta'=\beta/T$, as defined in Algorithm~\ref{alg:kl-ppm}.
\end{theorem}
Note the term $\frac{\sqrt{T d\log{(n^2\log^2{(1/\beta')}/d\beta')}}}{n\sqrt{\rho}}$ in the radius bound will be $o(1)$ in regime where the convergence guarantees are nontrivial. %
The privacy of Algorithm \ref{alg:kl-ppm} is straightforward since the subroutine $\cA$ is $\rho$-zCDP by the assumption. Algorithm \ref{alg:kl-ppm} is then private by post processing and composition. %
To prove the convergence result, we will use the following fact about the strength of differentially private optimizers for strongly convex loss functions.
\begin{lemma} \label{lem:dp-sc-whp}
There exists an implementation of $\cA$ which is $\rho$-zCDP and with probability at least $1-\beta'$ the output of the algorithm has excess risk  $O\br{\frac{L_0^2d\log{(n^2\log^2{(1/\beta')}/d\beta')}}{\tilde L_1n^2\rho}}$.
\end{lemma}
We provide the details for this result in Appendix \ref{app:dp-sc-whp}.
Furthermore, as in Section \ref{sec:kappa-leq-2}, we only need the KL condition to hold over the trajectory of the algorithm. The following lemma allows us to utilize the KL property at every iterate generated by the algorithm.
\begin{lemma} \label{lem:kl-ppm-trajectory-bound}
\mnote{rephrased}%
Assume $\cA$ is as described by Lemma \ref{lem:dp-sc-whp} above. With probability at least $1-\beta$, 
$w_1,...,w_T\in \ball{w_0}{\rad}$
for some $\rad=\tilde{O}\big(\frac{\lip}{\wc}\big(1+\frac{\sqrt{T d\log{(n^2\log^2{(1/\beta')}/d\beta')}}}{n\sqrt{\rho}}\big)+\lip\fnot^{\frac{\kappa-2}{\kappa}}\gamma^2\big)$. %
\end{lemma}
The proof is deferred to Appendix \ref{app:kl-ppm-trajectory}. The proof of Theorem \ref{thm:kappa-geq-2-result} now follows similar steps to those used in Theorem \ref{thm:kl-spider-convergence}, but is overall much simpler. One key difference is that, for each $t\in[T]$, we need to use the KL condition to lower bound $\|w_t - w_{t-1}\|$, rather than $\|\nabla F(w_t;S)\|$. For this, note that the optimality conditions of $F_t$ imply
$2\wc\|w_t^* - w_{t-1}\| = \|\nabla F(w_t^*;S)\| \geq \frac{1}{\gamma}(F(w_t^*;S)-F(w^*;S))^{1/\kappa}$.
The inequality comes from the KL condition. The full proof of Theorem \ref{thm:kappa-geq-2-result} is in Appendix \ref{app:kappa-geq-2-result}.

%% file: sections/adaptive-gd.tex
\section{Adapting to KL condition} \label{sec:adaptive-gd}
In this section, we present an alternative algorithm for ERM under the KL* condition. At the cost of weaker rates when $\kappa < 2$, our algorithm automatically adapts to $\kappa$ and $\gamma$. This is in contrast to the Spider method presented previously which requires prior knowledge of $\kappa$
and $\gamma$. Furthermore, we are able to obtain this result with a comparatively simple algorithm. That is, our algorithm is a simple modification of the traditional noisy gradient descent algorithm seen frequently in the DP literature \citep{bassily2014private,wang2017differentially,bassily2019private}.
\begin{algorithm}[h]
\caption{Adaptive Noisy Gradient Descent}
\label{alg:whp-ogd}
\begin{algorithmic}[1]
\REQUIRE Dataset $S$, Privacy parameter $\rho > 0$, Probability $\beta > 0$, Initial point $w_0\in\re^d$, Lipschitz parameter $\lip$, Smoothness parameter $\smooth$

\STATE $\eta = \frac{1}{2\smooth}$, ~$t = 0$, ~$\rho_0 = 0$

\WHILE{$\sum_{j=0}^t\rho_t \leq \frac{\rho}{2}$}

\STATE $N_t = \norm{\frac{1}{n}\sum_{i=1}^n \nabla f(w_{t};x_i)} + \hat{b}_t$ where $\hat{b}_t \sim \cN(0,\mathbb{I}_d\hat{\sigma}^2)$ and $\hat{\sigma} = \frac{\lip}{\sqrt{n}\rho^{1/4}}$ 

\STATE $\nabla_t = \frac{1}{n}\sum_{i=1}^n \nabla f(w_{t};x_i) + b_t$ where $b_t \sim \cN(0,\mathbb{I}_d\sigma_t^2)$ and $\sigma_t = \max\bc{\frac{N_t}{\sqrt{d\log(n\sqrt{\rho}/\beta)}},\frac{2\lip}{n\sqrt{\rho}}}$ 

\STATE $w_{t+1} = w_{t} - \eta \nabla _t$

\STATE $\rho_t = \min\bc{\frac{\lip^2 d\logterms}{n^2 N_t^2},\frac{\rho}{2}} + \frac{\sqrt{\rho}}{n}$

\STATE $t = t+1$

\ENDWHILE
\end{algorithmic}
\end{algorithm}
Throughout the following, we will use $T+1$ to denote the highest value of $t$ reached during the run of Algorithm \ref{alg:whp-ogd}.
\begin{theorem} \label{thm:adpative-gd-privacy}
Assume $f$ is Lipschitz. Then running Algorithm \ref{alg:whp-ogd} and releasing $w_0,...,w_T$ is $\rho$-zCDP.
\end{theorem}
The proof is given in Appendix \ref{app:adaptive-gd-privacy}, and relies on the fully adaptive composition theorem of \cite{Whitehouse2022FullyAC}.
Our aim is now to provide convergence guarantees when the loss satisfies the KL* condition over some region $\cS\subset\re^d$. We here demonstrate an alternative way of defining $\cS$ which allows us to leverage the KL* condition (in contrast to assuming $\cS$ is a ball).
Define the threshold 
$\alpha = \max\big\{F(w_0;S) ,F(w^*;S)+2(\gamma^{\kappa/2}+\lip)\br{\frac{\lip^2\logterms}{n\sqrt{\rho}}}^{1/2}\big\}$.
Let $\cI = \{w: F(w;S) \leq \alpha \}$ %
denote a lower level set of $F(\cdot;S)$. 
Note the second term in the max of $\alpha$ only handles the trivial case where $w_0$ already has small excess risk.
Observe that $\cI$ may not be a path-connected set, thus we define $\cS$ as the path-connected component of $\cI$ that contains $w_0$. 
That is, $w'\in \cS$ if there exists a continuous function $\mathbf{w}:[0,1]\to\cI$, such that $\mathbf{w}(0)=w_0$, $\mathbf{w}(1)=w'$. %
Intuitively, $\cS$ is the local ``valley'' of $F(\cdot;S)$ in which $w_0$ resides. \
Furthermore, we can guarantee that the trajectory of Algorithm \ref{alg:whp-ogd} stays in this valley for the duration of its run.
\begin{lemma} \label{lem:adaptive-gd-kl-star}
Assume $F(\cdot;S)$ is $\smooth$-smooth and $\lip$-Lipschitz. If $F(\cdot;S)$ satisfies the $(\gamma,\kappa)$-KL* condition over $\cS$ w.r.t. $w^*_\cS := \argmin\limits_{w\in\cS}\bc{F(w;S)}$, then w.p. at least $1-2\beta$, for all $t\in[T]$, $w_t \in \cS$.
\end{lemma}
The proof is deferred to Appendix \ref{app:adaptive-gd-trajectory}. Note we are assuming the KL* condition w.r.t. the minimizer over $\cS$ (as opposed to the global minimizer) here. %
We also remark that an existing work, \cite{Ganesh2023WhyPublic}, argued the importance of public pretraining in the non-convex setting to find some $w_0$ in a convex subregion before training on private data. Alternatively, our result suggests meaningful convergence if the empirical loss over  the localized region is instead KL. This may be more realistic in the overparameterized regime as existing work has shown such models tend to be non-convex (but KL) around the minimizer \citep{LB21}.
Our convergence result for Algorithm \ref{alg:whp-ogd} is as follows.
\begin{theorem} \label{thm:kl-ogd}
Let $\beta,\gamma>0$, $\kappa \in [1,2]$. Let $\rho\geq0$ be s.t. $\lip^2 \logterms/(\smooth n)\leq \sqrt{\rho}$. Define 
$p_{\mathsf{max}} := (1+8\gamma^2\smooth)\bs{\log(\fnot) + \frac{2\kappa}{4-\kappa}\log(n\sqrt{\rho}/[\gamma \lip])}$. 
If $F(\cdot;S)$ is $L_1$-smooth and $\lip$-Lipschitz and satisfies the $(\gamma,\kappa)$-KL* condition over $\cS$ (as described above) w.r.t. $w^*_{\cS}$, then with probability at least $1-2\beta$, Algorithm \ref{alg:whp-ogd} finds $w_T$ such that $F(w_{{T}};S) - F(w^*_\cS;S)$ is at most
{\small
\begin{align*}
     O\Bigg(\br{\frac{\gamma \lip \sqrt{ d \logterms p_{\mathsf{max}}}}{n\sqrt{\rho}}}^{\frac{2\kappa}{4-\kappa}}
     +\br{\frac{\max\bc{\gamma^2,1}\lip^2 \logterms}{\min\{\smooth,1\} n\sqrt{\rho}}}^{\kappa/2} 
     + \br{\frac{p_{\mathsf{max}}}{n\sqrt{\rho}}}^{\frac{\kappa}{2-\kappa}}\Bigg).
\end{align*}}
Ignoring polylogarithmic terms and problem constants %
we can more simply write
\linebreak 
$F(w_{{T}};S) - F(w_\cS^*;S) = \tilde{O}\big(\big(\frac{\sqrt{d}}{n\sqrt{\rho}}\big)^{\frac{2\kappa}{4-\kappa}} + \big(\frac{1}{n\sqrt{\rho}}\big)^{\kappa/2}\big).$
\end{theorem}
The simplification in the theorem uses the fact that $\frac{\kappa}{2-\kappa} \geq \frac{\kappa}{2}$ for all $\kappa$. %
We defer the proof of Theorem \ref{thm:kl-ogd} to Appendix \ref{app:adaptive-gd-kl}. The overarching ideas of the proof are similar to those of Theorem \ref{thm:kl-spider-convergence}. However, the adaptive nature of the algorithm makes the analysis much more delicate.

Observe that for $\kappa=2$ (i.e. the PL condition) this obtains the rate $\tilde{O}\big(\frac{d}{n^2\rho} + \frac{1}{n\sqrt{\rho}}\big)$ which essentially captures the optimal rate in this setting. The rate slows as $\kappa$ decreases, and for $\kappa = 1$ we obtain a rate of $\tilde{O}\big(\big(\frac{\sqrt{d}}{n\sqrt{\rho}}\big)^{2/3} + \frac{1}{\sqrt{n}\rho^{1/4}}\big)$. %

\subsection{Convergence Guarantees without the KL Condition} \label{sec:ogd}
One of the key properties of Algorithm \ref{alg:whp-ogd} is that it leverages large gradients to better control the privacy budget. In fact, even in the absence of an explicit KL assumption, we can show that Algorithm \ref{alg:whp-ogd} obtains strong convergence guarantees when large gradient norms are observed. We provide the following result on Adaptive Gradient Descent's ability to approximate stationary points. Note that we cannot give excess risk guarantees in this case due to the fact finding approximate global minimizers of non-convex functions is intractable in this setting.

\begin{theorem}\label{thm:ogd}
\sloppy
Assume $f$ is $\smooth$-smooth and $\lip$-Lipschitz. 
Let $T+1$ denote the largest value attained by $t$ during the run of Algorithm \ref{alg:whp-ogd}.
Let $t^*$ be sampled from $\bc{0,...,T}$ with probability proportional to $\exp\br{-\frac{n\sqrt{\rho}}{2\lip}\|\nabla F(w_t;S)\|}$. This algorithm is $2\rho$-zCDP and with probability at least $1-3\beta$ satisfies
\ifarxiv
\begin{align*}
    \norm{\nabla F(w_{t^*};S)}= O\br{\min\bc{\sqrt{\frac{\fnot \smooth}{T}},\br{\frac{\lip\sqrt{\fnot\smooth d}}{n\sqrt{\rho}}}^{1/2}} + \frac{\lip\sqrt{\logterms}}{\sqrt{n}\rho^{1/4}}}
\end{align*}
\else
$\norm{\nabla F(w_{t^*};S)}= O\br{\min\bc{\sqrt{\frac{\fnot \smooth}{T}},\br{\frac{\lip\sqrt{\fnot\smooth d}}{n\sqrt{\rho}}}^{1/2}} + \frac{\lip\sqrt{\logterms}}{\sqrt{n}\rho^{1/4}}}$.
\fi
\end{theorem}
The proof is given in Appendix \ref{app:ogd}.
The best case scenario is when most gradients in the run of the algorithm are $\Omega(1)$. In this case, the algorithm attains $T = \tilde{\Theta}\big(\min\big\{n\sqrt{\rho}, \frac{n^2\rho}{d}\big\}\big)$ iterations %
and the convergence guarantee becomes $\tilde{O}\big(\frac{\sqrt{d}}{n\sqrt{\rho}} + \frac{1}{\sqrt{n}\rho^{1/4}}\big)$. We note an existing work showed a lower bound $\Omega\big(\frac{\sqrt{d}}{n\epsilon}\big)$ for approximating stationary points, although this is not directly comparable as the previously stated upper bound does not hold for all functions.
In the worst case, the algorithm will achieve convergence guarantee $\tilde{O}\big(\frac{d^{1/4}}{\sqrt{n}\rho^{1/4}} + \frac{1}{\sqrt{n}\rho^{1/4}}\big)$. By contrast, the best known rate for approximating stationary points is $\tilde{O}\big(\big(\frac{\sqrt{d}}{n\sqrt{\rho}}\big)^{2/3}\big)$  \citep{ABGGMU23}, and the best known rate for methods which do not rely on variance reduced gradient estimates (as is more typical in practice) is $\tilde{O}\big(\big(\frac{\sqrt{d}}{n\sqrt{\rho}}\big)^{1/2}\big)$ \citep{wang2017differentially}. Our analysis recovers the $\tilde{O}\big(\big(\frac{\sqrt{d}}{n\sqrt{\rho}}\big)^{1/2}\big)$ rate obtained by noisy gradient descent as a worst case guarantee with minimal modification to the algorithm itself, while also potentially achieving a much stronger rate.

The worst case guarantee comes from balancing the number of iterations that the algorithm performs (which increases when the gradient norms are large) with the minimum gradient norm over the trajectory. For simplicity, consider the scenario where every gradient in the trajectory has the same norm $N>0$. Then clearly the minimum norm is also $N$, but in this case $T = \tilde{O}\big(\frac{N^2 n^2\rho}{d}\big)$. Thus the convergence guarantee implies that $N = \tilde{O}\big(\frac{\sqrt{d}}{N n\sqrt{\rho}}\big)$,  which at worst means $N=\tilde{O}\big(\frac{d^{1/4}}{\sqrt{n}\rho^{1/4}}\big)$. More formal/general details are in the proof in Appendix \ref{app:ogd}.

%% file: sections/appendix.tex
\newpage
\appendix

\input{sections/app-intro}

\input{sections/appendix-kappa-leq-2}

\input{sections/appendix-kappa-geq-2}

\input{sections/appendix-adaptive-gd}

\input{sections/appendix-reg-lip-opt}

%% file: sections/app-intro.tex
\section{Relationship between Growth Condition and KL Condition} \label{app:intro}
\begin{definition}[$(\lambda, \tau)$-growth] A function $F: \bbR^d\rightarrow\bbR$ satisfies $(\lambda, \tau)$-growth  if the set of minimizers $\cW^* :=\arg\min_w F(w)$ is non-empty, and
\begin{align*}
    F(w) - F(w_p) \geq \lambda^\tau \norm{w-w_p}^\tau
\end{align*}
where $w_p$ be the projection of $w$ onto $\cW^*$.
\end{definition}

\begin{lemma}[Theorem 5.2 (ii) in \cite{bolte2017error}]
\label{lem:growth-implies-kl}
Let $\kappa\geq 1$ and $\gamma>0$.
If $F:\bbR^d \rightarrow \bbR$ is convex and satisfies 
$\br{\gamma^{-1}, \frac{\kappa}{\kappa-1}}$
growth condition,
then it satisfies $(\gamma, \kappa)$-KL condition.
\end{lemma}

It is proven in \cite[Appendix A]{KN16} that the KL condition with $\kappa=2$ (i.e., the PL condition) implies quadratic growth. We present the following generalized version of this argument.

\begin{lemma} \label{lem:kl-implies-growth}
Assume $F:\re^d\to\re$ satisfies the $(\gamma,\kappa)$-KL condition for $\kappa \geq 1$ and $\gamma > 0$. Let $w\in\re^d$ and let $w_p$ be the projection of $w$ onto the set of optimal solutions, $\cW^* :=\arg\min_w F(w)$. 
Then it holds that 
    $F(w) - F(w_p) \geq \Big[\frac{1}{\gamma}\cdot\frac{\kappa-1}{\kappa}\Big]^{\frac{\kappa}{\kappa-1}}\norm{w-w_p}^{\frac{\kappa}{\kappa-1}}$. 
\end{lemma}
\begin{proof}
Define $F^* = \min\limits_{w\in\re^d}\bc{F(w)}$ and $g(w) = \frac{1}{1-1/\kappa}[F(w)-F^*]^{1-\frac{1}{\kappa}}$. We have
\begin{align} \label{eq:gnorm-bound}
    \| \nabla g(w) \|^2 &= \norm{\frac{\nabla F(w)}{[F(w) - F^*]^{1/\kappa}}}^2 \\
    &= \frac{\norm{\nabla F(w)}^2}{[F(w)-F^*]^{2/\kappa}} \\
    &= \br{ \frac{\norm{\nabla F(w)}^\kappa}{[F(w)-F^*]}}^{2/\kappa} 
    \geq \frac{1}{\gamma^2}
\end{align}

Consider the gradient flow starting at a point $w_0$ given by
\begin{align*}
    &\frac{d \boldsymbol{w}(t)}{dt} = -\nabla g(w(t)), 
    &\left. \boldsymbol{w}(t)\right|_{t=0} = w_0
\end{align*}
Note $F$ is invex (i.e. its 
stationary points are global minimizers) because it is KL. Thus $g$ is an invex function because it is the composition of monotonically increasing function and invex function. Further, because $g$ is  bounded from below (by $0$), the path described above eventually reaches the minimum thus there exists $T<+\infty$ such that $F(\boldsymbol{w}(T)) = F(w^*)$. %

Note the length of the path is at least $\|w_0 - w_p\|$. We then have
\begin{align*}
    g(w_0) - g(w_T) 
    &=-\int_{0}^T \ip{\nabla g(\boldsymbol{w}(t))}{\frac{d\boldsymbol{w}(t)}{dt}} dt \\
    &= \int_0^T \|\nabla g(\boldsymbol{w}(t))\|^2 dt \\
    &\overset{(i)}{\geq} \frac{1}{\gamma} \int_0^T \|\nabla g(\boldsymbol{w}(t))\| \\
    &\overset{(ii)}{\geq} \frac{1}{\gamma}\|w_0 - w_p\|,
\end{align*}
where $(i)$ uses Eqn. \eqref{eq:gnorm-bound} and $(ii)$ uses the lower bound on the path length.
Plugging in the definition of $g$ then gives  
\begin{align*}
    F(w) - F(w_p) \geq \bs{\frac{1-1/\kappa}{\gamma}  \|w-w_p\| }^{\frac{1}{1-1/\kappa}}.
\end{align*}
Note the bound is non-negative if $\kappa \geq 1$.
Finally, observing that ${\frac{1}{1-1/\kappa}} = \frac{\kappa}{\kappa-1}$ establishes the claim. %
\end{proof}

\begin{remark} \label{rem:min-in-ball}
Using the above result one can observe that if the KL condition holds over a ball of radius $\rad \geq \frac{\kappa}{\kappa-1}\gamma \fnot^{\frac{\kappa-1}{\kappa}}$, then $w^*\in \ball{w_0}{\rad}$. Then for some $w'\in\re^d$, a triangle inequality can then be used to obtain $\|w'-w^*\| \leq \|w_0-w'\| + \|w_0-w^*\|$. This would allow one to phrase our results in terms of a ball centered at $w^*$.
\end{remark}

%% file: sections/appendix-kappa-leq-2.tex
\section{Missing Proofs from Section \ref{sec:kappa-leq-2}} \label{app:kappa-leq-2}

\subsection{Privacy of Algorithm \ref{alg:kl-spider}} \label{app:kl-spider-privacy}
\begin{lemma} \label{lem:kl-spider-privacy}
Assume $f$ is $\lip$-Lipschitz and $\smooth$-smooth over $\ball{w^*}{\rad}$ (where $\rad$ is as given in Theorem \ref{thm:kl-spider-convergence}).
Then Algorithm \ref{alg:kl-spider} is $2\rho$-zCDP.
\end{lemma}
\begin{proof}
First, by Lemma \ref{lem:kl-spider-trajectory-bound}, every $w_{k,t}$, $k\in[K]$,$t\in[T_k]$, is in $\ball{w^*}{\rad}$, and thus the loss is Lipschitz and smooth at the iterates generated by the algorithm. 
The sensitivity of the minibatch gradient estimates (made in the outer loop) is then $\frac{\lip}{n}$, and at most $K$ such estimates are made. Smoothness guarantees the sensitivity of the gradient difference estimates (made in the inner loop) at some $k\in[T]$, $t\in[T_k]$ is %
$\frac{\eta_{k,t}\smooth}{n}\|\nabla_{k,t}\| \leq \frac{1}{n}\hat{\Phi}_k^{1/\kappa}$ since $\eta_{k,t} = \frac{1}{4\gamma\smooth \|\nabla_{k,t}\|}\hat{\Phi}_k^{1/\kappa}$. Note at most $T_k$ such estimates are made. 

The zCDP guarantees of the Gaussian mechanism ensures that the process of generating each $\nabla_{k,0}$ is $\hat{\rho}$-CDP with $\hat{\rho} = \frac{1}{K}$. Similarly, we have that the process of generating each $\Delta_{k,t}$, $t>0$, is $\rho_k$-zCDP with $\rho_k=\frac{\rho}{K T_k}$. By the composition theorem for zCDP we then have the overall privacy, is at most
$\sum_{k=1}^K\br{\frac{\rho}{K} + \sum_{t=1}^{T_k}\frac{\rho}{K T_k}} = 2\rho$.
\end{proof}

\subsection{Descent Equation for Algorithm \ref{alg:kl-spider}} \label{app:kl-spider-step-convergence}
\begin{lemma} \label{lem:kl-spider-step-convergence}
With probability at least $1-\beta$, for every $k\in[K]$ and $t\in[T_k]$ indexing iterates of the algorithm it holds that 
$F(w_{k,t};S) - F(w_{k,t+1};S) \geq \frac{1}{16\gamma\smooth}\|\nabla_{k,t}\|\hat{\Phi}_k^{1/\kappa} - \frac{1}{4\smooth}\|\nabla_{k,t} - \nabla F(w_{k,t};S)\|^2$ %
\end{lemma}
\begin{proof}
We start with a standard descent analysis. Since $F(\cdot;S)$ is $\smooth$-smooth, we have
\begin{align*}
    F(w_{k,t};S) - F(w_{k,t+1};S) &\geq \ip{\nabla F(w_{k,t};S)}{w_{k,t} - w_{k,t+1}} - \frac{\smooth}{2}\norm{w_{k,t+1}-w_{k,t}}^2 \\
    &= \eta_{k,t}\ip{\nabla F(w_{k,t};S)}{\nabla_{k,t}} - \frac{\smooth \eta_{k,t}^2}{2}\norm{\nabla_{k,t}}^2 \\
    &= \eta_{k,t}\br{1 - \frac{\eta_{k,t} \smooth}{2}}\|\nabla_{k,t}\|^2 + \eta_{k,t} \ip{\nabla F(w_t;S) - \nabla_{k,t}}{\nabla_{k,t}} \\
    &\overset{(i)}{\geq} \eta_{k,t}\br{\frac{1}{2} - \frac{\eta_{k,t} \smooth}{2}}\|\nabla_{k,t}\|^2 - \frac{\eta_{k,t}}{2}\|\nabla_{k,t} - \nabla F(w_{k,t};S)\|^2. \\
    &\overset{(ii)}{\geq} \frac{\eta_{k,t}}{4}\|\nabla_{k,t}\|^2 - \frac{1}{4\smooth}\|\nabla_{k,t} - \nabla F(w_{k,t};S)\|^2 \\  
    &\overset{(iii)}{=} \frac{1}{16\gamma\smooth}\|\nabla_{k,t}\|\hat{\Phi}_k^{1/\kappa} - \frac{1}{4\smooth}\|\nabla_{k,t} - \nabla F(w_{k,t};S)\|^2 \\
\end{align*}    
Step $(i)$ uses Young's inequality. Step $(ii)$ uses the fact that $\eta_{k,t} \leq \frac{1}{2\smooth}$. This is because
$\eta_{k,t} = \frac{1}{4\gamma\smooth \|\nabla_{k,t}\|}\hat{\Phi}_k^{1/\kappa}$ and updates are only performed when $\|\nabla_{k,t}\| \geq \frac{7}{8\gamma}\hat{\Phi}_k^{1/\kappa}$. Step $(iii)$ uses the setting of $\eta_t$.
\end{proof}

\subsection{Proof of Lemma \ref{lem:kl-spider-trajectory-bound}} \label{app:kl-spider-trajectory}

Due to the step size and the phases lengths, with probability 1, we have
that,
\begin{align*}
    \|w_{k,t} - w_0\| &\leq \sum_{k=1}^K \frac{1}{4\gamma\smooth}\hat{\Phi}_k^{1/\kappa} T_k 
    &\leq \sum_{k=1}^K \frac{1}{4\fnot^{\frac{\kappa-2}{\kappa}}\gamma\smooth}\hat{\Phi}_k^{1/\kappa} \hat{\Phi}_k^{\frac{\kappa-2}{\kappa}} 
    &= \frac{\fnot^{\frac{2-\kappa}{\kappa}}\fnot^{\frac{\kappa-1}{\kappa}}}{4\gamma\smooth}\sum_{k=1}^K \br{\frac{1}{c^{\frac{\kappa-1}{\kappa}}}}^k
\end{align*}
Above, we use the fact that $\frac{\kappa-1}{\kappa}\geq0$ (since $\kappa\geq 1$) to bound $\hat{\Phi}_k \leq \fnot$.
Since $c>1$ we have, recalling $K=(1+64\fnot^{\frac{\kappa-2}{\kappa}}\gamma^2\smooth)\Big[\log(\fnot) -\kappa\log\Big(\frac{\lip\sqrt{d}}{n\sqrt{\rho}}\Big)\Big]$,
\begin{align*}
    \|w_{k,t} - w_0\| \leq \frac{K\fnot^{1/\kappa}}{4\gamma\smooth} 
    = \br{\frac{\fnot^{1/\kappa}}{4\gamma\smooth} + 16\fnot^{\frac{\kappa-1}{\kappa}}\gamma}\Big[\log(\fnot) + \kappa\log\Big(\frac{n\sqrt{\rho}}{\lip\sqrt{d}}\Big)\Big].
\end{align*}

\subsection{Lower Bound for Smooth Losses Satisfying Growth Condition} \label{app:asi-lb-extension}
We provide the following extension of the lower bound on excess risk in \cite{ALD21}. Our extension yields a lower bound for losses which satisfy $(\lambda,\tau)$-growth and are $\smooth\geq0$-smooth over a ball $\ball{w_0}{R}$, for any smoothness parameter $\smooth\geq0$ and radius $R>0$.
In contrast, the setting of \cite{ALD21} did not have the above smoothness and existence of a  (large) ball $\ball{w_0}{R}$ assumption (over which smoothness and Lipschitzness holds). 
Further, \cite{ALD21} provide a lower bound for \textit{constrained} DP procedures, which is
is based on a reduction from convex ERM over a constrained set of \textit{any} diameter $D$ \citep{bassily2014private}.
In contrast, we are interested in 
lower bound for unconstrained procedures.
Therefore, in Theorem \ref{thm:unc_lb}, we extend the lower bound of \cite{bassily2014private} to the unconstrained setting.
We then provide a reduction, closely following \cite{ALD21}, from unconstrained convex ERM to unconstrained optimization of functions satisfying a growth condition. %
Finally, we note that our unconstrained lower bound in Theorem \ref{thm:unc_lb} holds pointwise for all values of the norm of optimal solution $D$, so it suffices to construct a reduction for \textit{some} choice of $D$.
We show that for any given setting of problem parameters, there is a choice of $D$, for which the reduced instance satisfies the requisite properties.

\begin{theorem}
\label{thm:lb_growth}
    Let $L_0, L_1, \rad, \lambda \geq 0, \tau\geq 2, \tau = O(1), 0 <\epsilon \leq 1, 
    2^{-\Omega(n)}\leq \delta < \frac{1}{n} $. For any $(\epsilon,\delta)$-DP algorithm $\cA$, there exists a set $\cW \subset \bbR^d$ containing a ball of radius $\rad$,
    a dataset $S$ and a convex loss function $f$ such that for all $x$, the function $w\mapsto f(w;x)$ is $L_0$-Lipschitz, $L_1$-smooth over $\cW$, 
    the empirical loss $w\mapsto F(w;S)$ satisfies $(\lambda,\tau)$-growth,
    and
   \begin{align*}
        \mathbb{E}_{\cA}[F(\cA(S);S) - \inf_{w \in \bbR^d}F(w)]  = \Omega\br{\frac{1}{\tau^{\frac{1}{\tau-1}}}\br{\frac{ L_0\sqrt{d}}{\lambda n\epsilon}}^\frac{\tau}{\tau-1}}.
    \end{align*}
\end{theorem}
\begin{proof}
 The key to the proof is the following reduction, based on Proposition 3 of \cite{ALD21}. 
 Herein, the aim is  
 to show that the existence of a DP optimizer for convex losses satisfying the growth condition implies the existence of an optimizer for general convex losses.
      More formally, consider a problem instance class where we are given a
       set $\cW \subset \bbR^d$ containing a ball of radius $\rad$, a dataset $S \in \cX^n$ for some $\cX$, a function $f(w;x)$ where $w\mapsto f(w;x)$ is $L_0$-Lipschitz, $L_1$-smooth over $\cW$ 
       for all $x \in \cX$ and the empirical loss $w\mapsto F(w;S)$ satisfies
      $(\lambda,\tau)$-growth.
      Note since these properties hold over $\cW$, they hold over the ball of radius $\rad$.
      If there exists an $(\epsilon, \delta)$-DP algorithm $\cA$, which for the above problem instance has expected excess empirical risk,
    \begin{align*}
        \mathbb{E}_{\cA}[F(\cA(S);S) - \inf_{w}F(w)] = o\br{\br{\tau\lambda^\tau}^{-\frac{1}{\tau-1}} \Delta(n,d, L_0, L_1, \epsilon,\delta)}, 
    \end{align*}
    then for 
    $D = \max\br{\frac{\br{\Delta(n,d, L_0, L_1, \epsilon,\delta)}^{1/\tau}L_1 L_0^{\frac{\tau-2}{\tau-1}}}{c_2(\tau)}, \frac{\br{\Delta(n,d, L_0, L_1, \epsilon,\delta)}^{1/\tau}\rad}{L_0^{\frac{1}{(\tau-1)}}c_3(\tau)}, \frac{\sqrt{d}L_0}{L_1n\epsilon}}$
    , where $c_2(\tau)=\Omega(1)$ and $c_3(\tau)=\Omega(1)$ are specified later,
 there exists  an $(\epsilon, \delta)$-DP algorithm $\tilde \cA$, such that for any $L_0$-Lipschitz, convex, $L_1$-smooth loss function $w\mapsto \tilde f(w;x)$ for all $x$, with minimizer norm $\norm{w^*} = D$,
its excess risk is 
    \begin{align*}
        \mathbb{E}_{\tilde \cA}[\tilde F(\tilde \cA(S);S) - \inf_{w\in \bbR^d}\tilde F(w)] = o\br{D\br{\Delta(n,d, 2L_0, 2L_1, \epsilon/k,\delta/k)}^{\frac{\tau-1}{\tau}}}
    \end{align*}  
    where $k$ is the smallest integer larger that %
$
\log\br{\frac{\tau^{\frac{1}{\tau-1}} L_0^{\frac{\tau}{\tau-1}}}{2^{2\tau-3}\Delta(n,d, 2L_0, 2L_1, \epsilon/k,\delta/k)}}
$.

    The main difference between above and the statement of \cite{ALD21} is that unlike \cite{ALD21}, our reduction is for unconstrained procedures and is tailored to the aforementioned choice of diameter $D$.

   The proof uses the construction of \cite{ALD21},
 verifying that for the provided parameter settings, the assumptions hold.
   For simplicity of notation, let $\Delta = \Delta(n, L_0, L_1, \epsilon,\delta)$.

    Let $w_0$ be the origin.
For a sequence of $\bc{\lambda_i}_i$ to be instantiated later, 
define 
    \begin{align*}
        &\tilde \cW_i = \bc{w: \norm{w-w_{i-1}} \leq \br{\frac{L_0}{2\tau \lambda_i^\tau 2^{\tau-2}}}^{\frac{1}{\tau-1}}}\\
        &\tilde F_i(w;S) = F(w;S) + \lambda_i^\tau 2^{\tau-2} \norm{w-w_{i-1}}^\tau
    \end{align*}
where $w_i = \cA(\tilde F_{i},S)$.
The function $\tilde F_i$ satisfies $(\lambda_i 2^{(\tau-2)/\tau},\tau)$-growth
(over all of $\bbR^d$). We now inspect its Lipschitzness and smoothness parameters over $\tilde \cW_i$. By direct calculation, the Lipschitz parameter is bounded by $L_0+L_0 = 2L_0$
The smoothness parameter is at most,
\begin{align*}
    L_1+\lambda_i^\tau 2^{\tau-2}\tau (\tau-1) \norm{w - w_{i-1}}^{\tau-2} &=L_1+\cnote{typo?}\enayat{yes, fixed} \lambda_i^\tau 2^{\tau-2}\tau (\tau-1) \br{\frac{L_0}{2\tau \lambda_i^\tau 2^{\tau-2}}}^{\frac{\tau-2}{\tau-1}} \\
    & =L_1+ \br{\lambda_i}^{\frac{\tau}{\tau-1}} (L_0)^{\frac{\tau-2}{\tau-1}} c_1(\tau),
\end{align*}
where $c_1(\tau) = \frac{2^{\frac{\tau-2}{\tau-1}}\tau^{\frac{1}{\tau-1}}(\tau-1)}{2^{\frac{\tau-2}{\tau-1}}} = \tau^{\frac{1}{\tau-1}}(\tau-1)$. In \cite{ALD21}, $\lambda_i$ is set as $\lambda_i = 2^{-\br{\frac{\tau-1}{\tau}}i}\lambda$ for $\lambda$ to be specified later. The above smoothness bound is a decreasing function in $i$, so what suffices is to show that the above bound is smaller than $2L_1$ for the largest $\lambda_i$, which is $\lambda_1 = 2^{-\br{\frac{\tau-1}{\tau}}}\lambda$. From \cite{asi2021private}, $\lambda = 4^{\frac{(\tau-1)^2}{\tau^2}}\br{\frac{\Delta \tau}{D^\tau (\tau-1)}}^{\frac{(\tau-1)}{\tau^2}}$, so we have, 
\begin{align*}
    \br{\lambda_i}^{\frac{\tau}{\tau-1}} (L_0)^{\frac{\tau-2}{\tau-1}} c(\tau)  = 4^{\frac{\tau-1}{\tau}}\br{\frac{\tau}{\tau-1}}^{1/\tau}\frac{\Delta^{1/\tau}}{D}  (L_0)^{\frac{\tau-2}{\tau-1}} c_1(\tau) = c_2(\tau) \frac{\Delta^{1/\tau}}{D}  (L_0)^{\frac{\tau-2}{\tau-1}},
\end{align*}
where $c_2(\tau) = 4^{\frac{\tau-1}{\tau}}\br{\frac{\tau}{\tau-1}}^{1/\tau} \tau^{\frac{1}{\tau-1}}(\tau-1)$. The choice of $D \geq \frac{L_1 \Delta^{1/\tau}(L_0)^{\frac{\tau-2}{\tau-1}}}{c_2(\tau)}$, ensures the above is at most $L_1$, thereby establishing that the smoothness parameter is at most $2L_1$. The final condition we want to ensure is that all the sets $\tilde \cW_i$ contain a ball of radius at least $\rad$.
 Since $\lambda_i$ is decreasing in $i$, 
it suffices to consider $i=1$. We have,
\begin{align*}
    \br{\frac{L_0}{2\tau \lambda_1^\tau2^{\tau-2}}}^{\frac{1}{\tau-1}} = \frac{1}{2\tau^{\frac{1}{\tau-1}}}\br{\frac{\tau-1}{\tau}}^{\frac{1}{\tau}}\frac{1}{4^{\frac{\tau-1}{\tau}}}L_0^{\frac{1}{\tau-1}}\frac{D}{\Delta^{1/\tau}} = c_3(\tau) L_0^{\frac{1}{\tau-1}}\frac{D}{\Delta^{1/\tau}} 
\end{align*}
where $c_3(\tau) = \frac{1}{2\tau^{\frac{1}{\tau-1}}}\br{\frac{\tau-1}{\tau}}^{\frac{1}{\tau}}\frac{1}{4^{\frac{\tau-1}{\tau}}}$. The choice of $D \geq \frac{\rad k \Delta^{1/\tau}}{L_0^{\frac{1}{\tau-1}}c_3(\tau)}$ ensures the above is at least $\rad$.

The rest of the proof repeats the arguments in \cite{ALD21}, to get, 
\begin{align*}
    \mathbb{E}[\tilde F(\cA'(S);S)] - \min_{w}\tilde F(w;S) 
    = o\br{D\br{\Delta(n,d, 2L_0, 2L_1, \epsilon/k,\delta/k)}^\frac{\tau-1}{\tau}}
\end{align*}
We now instantiate 
$\Delta(n,d,L_0,L_1, \epsilon,\delta) = \frac{L_0\sqrt{d}}{n\epsilon}$.
This gives us that $\mathbb{E}[\tilde F(\cA'(S);S)] - \min_{w}\tilde F(w;S) 
    = o\br{\frac{L_0D\sqrt{d}}{n\epsilon}}$. However, this contradicts our lower bound in Theorem \ref{thm:unc_lb} for unconstrained DP procedures for convex, $L_0$-Lispchitz, $\frac{\sqrt{d}L_0}{Dn\epsilon} \leq L_1$-smooth (by our choice of $D$) losses.
\end{proof}

\subsection{Additional Details for Corollary \ref{cor:kl-lb} (Lower Bound)} \label{app:kl-lb}
In Theorem \ref{thm:lb_growth} (in Appendix \ref{app:asi-lb-extension}), an extension of \cite[Theorem 6]{ALD21},  we show that for $\tau\geq 2$ and $\tau = \Theta(1)$,
the lower bound on the minimax optimal expected excess empirical risk, $\alpha$, 
for $(\epsilon,\delta)$-DP ERM of functions which are smooth and Lipschitz over a ball of any finite radius $\rad>0$ and globally satisfy convexity and $(\lambda,\tau)$-growth, is
\begin{align*}
    \alpha 
    = \tilde{\Omega}\br{\frac{1}{(\tau)^{\frac{1}{\tau-1}}}\br{\frac{\lip\sqrt{d}}{\lambda n\epsilon}}^{\frac{\tau}{\tau-1}}} .
\end{align*}
Lemma \ref{lem:growth-implies-kl} gives that $(\lambda,\tau)$-growth and convexity implies $(\gamma,\kappa)$-KL with $\lambda = \gamma^{-1}$ and 
$\tau = \frac{\kappa}{\kappa-1}$.  Further, if $\kappa \leq 2$, then $\tau=\frac{\kappa}{\kappa-1} \geq 2$ and if $\kappa = 1+\Omega(1)$ then $\tau = O(1)$. Thus we have the lower bound,
\begin{align*}
    \alpha &= \tilde{\Omega}\br{\bs{\frac{\kappa}{\kappa-1}}^{1-\kappa} \br{\frac{\gamma \lip\sqrt{d}}{n\epsilon}}^{\kappa}} 
    = \tilde{\Omega}\br{\br{\frac{\gamma\lip\sqrt{d}}{n\epsilon}}^{\kappa}}.
\end{align*}
The last step uses the fact that $\kappa = 1+\Omega(1)$. 
Finally, the existence of a $\rho$-zCDP algorithm with rate better than $\tilde{O}\br{\br{\frac{\gamma\lip\sqrt{d}}{n\sqrt{\rho}}}^\kappa}$ would imply the existence of an $(\epsilon,\delta)$-DP algorithm (see \cite[Proposition 1.3]{Bun-zCDP}) with rate better than $\tilde{O}\br{\br{\frac{\gamma\lip\sqrt{d}}{n\epsilon}}^{\kappa}}$, a contradiction.

\subsection{Unconstrained Lower Bound for General Loss Functions} \label{app:general-lb}

In this section, we provide an extension of the lower bound on excess risk of DP procedures for convex Lipschitz functions in the constrained setting \citep{bassily2014private} to the unconstrained setting.
The key idea in the proof is to define a Lipschitz extension of the hard instance in \cite{bassily2014private} using the Huber regularizer. The dataset for our construction, as in \cite{bassily2014private}, leverages fingerprinting codes. The exact details of fingerprinting codes are note needed for our proof below, but we defer the interested reader to \cite{BUV14} for more details.
The following result is used in the proof of the lower bound for functions satisfying $(\lambda, \tau)$-growth for $2\leq \tau = O(1)$ in Theorem \ref{thm:lb_growth}.

\begin{theorem} \label{thm:unc_lb}
Let $0<\epsilon \leq 1, 0< \delta <\frac{1}{n}$, $D, L_0>0$.
For any $(\epsilon,\delta)$-DP algorithm, there exists a dataset $S$, and a $L_0$-Lipschitz, $\frac{\sqrt{d}L_0}{n\epsilon D}$-smooth convex loss function $w\mapsto f(w;x)$ for all $x$, such that its unconstrained minimizer, $w^* = \argmin\limits_{w}\bc{F(w;S)}$, has norm at most $D$, and 
$$\mathbb{E}_{\cA}[F(\cA(S);S) - F(w^*;S)] = \Omega\br{L_0D\min\bc{\frac{\sqrt{d}}{n\epsilon}, 1}}.$$
\end{theorem}
\begin{proof}
Consider the loss function
\begin{align}\label{eq:unc_lb_loss}
    F(w;S) = \frac{1}{n}\sum_{i=1}^n \ip{w}{x_i} + \lambda H(w),
\end{align}
where $H$ is the ``Huber regularization" defined as
\begin{align}
    H(w) = \begin{cases}
    \norm{w}^2 &\text{if} \norm{w} \leq 4D \\
    4D\norm{w} &\text{otherwise}  
    \end{cases}
\end{align}

Note that if $N$ of the $x_i$ vectors are 
vectors
in $\bc{\pm\frac{L_0}{\sqrt{d}}}^d$ and the rest are the zero vector, we have $\norm{\sum_{i=1}^n x_i} \leq NL_0$.
The empirical minimizer is $w^*= -\frac{\sum_{i=1}^n x_i}{2n\lambda}$. Thus we set $\lambda = \frac{NL_0}{2nD}$ so that $\norm{w^*} \leq D$. We also remark that under this setting of $\lambda$ that $F$ is Lipschitz with parameter $L_0' = L_0+4\lambda D \leq 5L_0$.

Now we will show that any $w$ which achieves small excess risk is close to $w^*$. Then we will use a lower bound on this distance to lower bound the error (as in \cite{bassily2014private}). For any $w$ such that $\norm{w} \leq 4D$ have
\begin{align*}
    F(w;S) - F(w^*;S) &= \ip{w-w^*}{\frac{1}{n}\sum_{i=1}^n x_i} + \lambda\br{\norm{w}^2 - \norm{w^*}^2} \\
    &= 2\lambda\ip{w-w^*}{-w^*} +  \lambda\br{\norm{w}^2 - \norm{w^*}^2} \\
    &= 2\lambda(\|w^*\|^2 -\ip{w}{w^*}) +  \lambda\br{\norm{w}^2 - \norm{w^*}^2} \\
    &= 2\lambda\br{\frac{1}{2}\|w^*\|^2 - \frac{1}{2}\norm{w}^2 + \frac{1}{2}\norm{w-w^*}^2} +  \lambda\br{\norm{w}^2 - \norm{w^*}^2} \\
    &= \lambda\norm{w-w^*}^2 \\
    &= \frac{NL_0}{nD}\norm{w-w^*}^2
\end{align*}
where the fourth equality comes from $\ip{a}{b} = \frac{1}{2}(\norm{a}^2 + \norm{b}^2 - \norm{a-b}^2)$. 
Now \cite[Lemma 5.1]{bassily2014private} gives that for $N=\min\bc{\frac{\sqrt{d}}{\epsilon},n}$ there exists a construction of the non-zero dataset vectors such that the
output of any $(\epsilon,\delta)$-DP algorithm, $\cA(S)$, must satisfy $\mathbb{E}[\norm{\cA(S) - w^*}] = \Omega\br{\frac{\sqrt{d}D}{N\epsilon}}$. Thus we have
\begin{align*}
     \mathbb{E}[F(\cA(S);S) - F(w^*;S)] &= \Omega\br{L_0D\min\bc{\frac{\sqrt{d}}{n\epsilon}, 1}}.
\end{align*}

This lower bounds the excess loss for any $w$ such that $\|w\| \leq 4D$. Finally, note that any $w'$ such that $\norm{w'} \geq 4D$ (i.e. a point outside the quadratic region of $H$) would also have high empirical risk because of the regularization term. Specifically, we have for any such $w'$ that
\begin{align*}
    F(w';S) &\geq -\frac{\|w'\|L_0N}{n} + 4\lambda D\|w'\| \geq 16\lambda D^2 - \frac{4DL_0N}{n}
\end{align*}
Further since $\norm{w^*}\leq D$, we have $F(w^*;S) \leq \frac{NL_0}{n} + \lambda D^2$. This gives
\begin{align*}
     F(w';S) - F(w^*;S) \geq 15 \lambda D^2 - \frac{5L_0DN}{n} = \Omega\br{\frac{\sqrt{d}D\lip}{n\epsilon}}
\end{align*}
where the last step follows from the setting of $\lambda$. Combining the two cases finishes the proof.
\end{proof}

%% file: sections/appendix-kappa-geq-2.tex
\section{Missing Results from Section \ref{sec:kappa-geq-2}} \label{app:kappa-geq-2}
\subsection{Proof of Lemma \ref{lem:kl-ppm-trajectory-bound}} \label{app:kl-ppm-trajectory}
\mnote{Added some extra explanation}
For any $t\in[T]$, the stationarity conditions of $F_t$ imply 
$\|\nabla F(w_t^*;S)\| = 2\wc\|w_t^* - w_{t-1}\|$, and so by Lipschitzness $\|w_t^* - w_{t-1}\| \leq \frac{\lip}{2\wc}$. 
Further, we have by strong convexity and the accuracy guarantee of $\cA$ that with probability $1-\beta$ for any $t\in[T]$ that %
$\|w_t - w_t^*\| = O\br{\frac{1}{\sqrt{\wc}}\sqrt{F_t(w_t;S) - F_t(w_t^*;S)}} = O\br{\frac{\lip\sqrt{T d\log{(n^2\log^2{(1/\beta')}/d\beta')}}}{\wc n\sqrt{\rho}}}$.
Thus using the triangle inequality the overall magnitudes of the updates are bounded by 
$\|w_t^* - w_{t-1}\| = O\br{\frac{1}{\wc}\br{\lip + \frac{\lip\sqrt{T d\log{(n^2\log^2{(1/\beta')}/d\beta')}}}{n\sqrt{\rho}}}}.$ 
In the following, let $\tau = \frac{\sqrt{T d\log{(n^2\log^2{(1/\beta')}/d\beta')}}}{n\sqrt{\rho}}$. Since at most $T$ iterations occur, we have

\begin{align*}
    \|w_t-w_0\| &= O\br{\frac{T}{\wc}\br{\lip + \frac{\lip\sqrt{T d\log{(n^2\log^2{(1/\beta')}/d\beta')}}}{n\sqrt{\rho}}}} \\
    &= O\br{\frac{T\lip(1+\tau)}{\wc}} \\
    &= O\br{\frac{\lip(1+\tau)}{\wc}(1+\fnot^{\frac{\kappa-2}{\kappa}}\gamma^2\wc)\Big[\log(\fnot) - \kappa\log\Big(\frac{\gamma \lip \sqrt{d\log(1/\beta')}}{n\sqrt{\rho}}\Big)\Big]} \\
    &= O\br{\Big(\frac{\lip(1+\tau)}{\wc}+\lip\fnot^{\frac{\kappa-2}{\kappa}}\gamma^2\Big)\Big[\log(\fnot) - \kappa\log\Big(\frac{\gamma \lip \sqrt{d\log(1/\beta')}}{n\sqrt{\rho}}\Big)\Big]}.
\end{align*}

\subsection{Proof of Theorem \ref{thm:kappa-geq-2-result} (Convergence of PPM under the KL* Condition)} \label{app:kappa-geq-2-result}

\begin{proof}[Proof of Theorem \ref{thm:kappa-geq-2-result}]
In the following, we condition on the event that every run of $\cA$ obtains excess risk at most $\frac{a L_0^2d\log{(n^2\log^2{(1/\beta')}/d\beta')}}{\tilde L_1n^2\rho}$ for some universal constant $a$. Since $\beta' = \frac{\beta}{T}$, this event happens w.p. at least $1-\beta$ by Lemma \ref{lem:dp-sc-whp}. Further, under this same event, the KL condition holds at every $w_{t}$, $t\in[T]$, by Lemma \ref{lem:kl-ppm-trajectory-bound}.

Now define $c=1+\fnot^{\frac{2-\kappa}{\kappa}}\frac{1}{32\gamma^2\smooth}$, $\hat{\Phi}_0 = \fnot$ and 
\begin{align*}
\hat{\Phi}_t = \max\Bigg\{\frac{1}{c}\hat{\Phi}_{t-1},\min\Big\{\Big(\frac{a\gamma \lip\sqrt{T d\log(n^2\log^2{(1/\beta')}/d\beta')}}{n\sqrt{\rho}}\Big)^{\kappa}, \fnot\Big\}\Bigg\}.
\end{align*}
We will first prove by induction that $F(w_t;S) - F(w^*;S) \leq \hat{\Phi}_t$ under the assumption that $F(w_{t-1};S) - F(w^*;S) \leq \hat{\Phi}_{t-1}$. Clearly the base case is satisfied for $\hat{\Phi}_0$.

To prove the induction step, we will proceed by contradiction. That is, assume by contradiction that $F(w_t;S) - F(w^*;S) > \hat{\Phi}_t$.

Note $F_t$ is $\wc$-strongly convex since it is the sum of a $\wc$ weakly convex function and a $2\wc$ strongly convex function \citep{DD19}. Let $\tau$ be an upper bound on the excess risk achieved by $\cA$ on the strongly convex objective $F_t$. Then 
\begin{align}
    F(w_t;S) = F_t(w_t;S) 
    &\leq F_t(w^*_t;S) + \tau 
    \overset{(i)}{\leq} F(w_{t-1};S) - \frac{\wc}{2}\norm{w_{t-1}-w_t^*}^2 + \tau \nonumber \\
    &\implies F(w_{t-1};S) - F(w_t;S) \geq \frac{\wc}{4}\norm{w_{t-1}-w_t^*}^2 - \tau  \label{eq:wc-descent}
\end{align}
Inequality $(i)$ uses the fact that $F_t(w^*_t;S) = F(w^*_t;S) + \frac{\wc}{2}\|w^*_t - w_{t-1}\|^2 \leq F_t(w_{t-1};S) = F(w_{t-1};S)$, which implies $F_t(w^*_t;S) \leq F(w_{t-1};S) - \frac{\wc}{2}\|w^*_t - w_{t-1}\|^2$.
Recall we have
$\tau \leq \frac{a L_0^2d\log{(n^2\log^2{(1/\beta')}/d\beta')T}}{\tilde L_1n^2\rho}$ by Lemma \ref{lem:dp-sc-whp}.
Further, note by stationarity conditions for the regularized objective we have
\begin{align}
    \|\nabla F(w_t^*;S)\| = 2\wc\|w_t^* - w_{t-1}\|. \label{eq:wc-stat-con}
\end{align}
By the assumption that $F(w_t^*;S) - F(w^*;S) \geq \hat{\Phi}_t$ and the KL condition we have
$\|\nabla F(w_t^*;S)\| \geq \frac{1}{\gamma}\hat{\Phi}_t^{1/\kappa}$, and thus by Eqn. \eqref{eq:wc-stat-con} we have
$\|w_t^* - w_{t-1}\| \geq \frac{1}{2\gamma\wc}\hat{\Phi}_t^{1/\kappa}$.
Applying Eqn. \eqref{eq:wc-descent} gives
\begin{align*}
    F(w_{t-1};S) - F(w_t;S) &\geq \frac{1}{16\gamma^2\wc}\hat{\Phi}_t^{2/\kappa} - \tau 
    \overset{(i)}{\geq} \frac{1}{32\gamma^2\wc}\hat{\Phi}_t^{2/\kappa}, %
\end{align*}
where inequality $(i)$ comes from the setting $\hat{\Phi}_t \geq \br{\frac{a \gamma \lip\sqrt{T d\log(n^2\log^2{(1/\beta')}/d\beta')}}{n\sqrt{\rho}}}^{\kappa}$.
Adding and subtracting $F(w^*;S)$ on the left hand side and rearranging obtains
\begin{align*}
    F(w_t;S) - F(w^*;S) &\leq F(w_{t-1};S) - F(w^*;S) - \frac{1}{32c^{2/\kappa}\gamma^2\wc}\hat{\Phi}_{t-1}^{2/\kappa} \\
    &\overset{(i)}{\leq} \hat{\Phi}_{t-1} - \frac{1}{32c^{2/\kappa}\gamma^2\wc}\hat{\Phi}_{t-1}^{2/\kappa} \\
    &= \Big(1-\Phi_{t-1}^{\frac{2-\kappa}{\kappa}}\frac{1}{32c^{2/\kappa}\gamma^2\wc}\Big)\hat{\Phi}_{t-1} \\
    &\overset{(ii)}{\leq} \Big(1-\fnot^{\frac{2-\kappa}{\kappa}}\frac{1}{32c\gamma^2\wc}\Big)\hat{\Phi}_{t-1} 
    = \frac{1}{c}\hat{\Phi}_{t-1} \leq \hat{\Phi}_t.
\end{align*}
Step $(i)$ uses the inductive assumption that $F(w_{t-1}) - F(w^*;S) \leq \hat{\Phi}_{t-1}$. 
Inequality $(ii)$ uses the fact that $\kappa\geq 2$, $c \geq 1$, and $\Phi_{t-1}\leq \fnot$. 
This establishes a contradiction and thus completes the induction argument.
We have now proven that $F(w_t;S) - F(w^*;S) \leq \hat{\Phi}_t$ for all $t\in \bc{0,...,T}$.

All that remains to prove convergence is to show that $\hat{\Phi}_T \leq \br{\frac{\gamma \lip\sqrt{T d\log(n^2\log^2{(1/\beta')}/d\beta')}}{n\sqrt{\rho}}}^{\kappa}$. We have
$\hat{\Phi}_T \leq \max\bc{\br{\frac{1}{c}}^T\fnot, \br{\frac{\gamma \lip\sqrt{T d\log(n^2\log^2{(1/\beta')}/d\beta')}}{n\sqrt{\rho}}}^{\kappa}}$ and
{\scriptsize
\begin{align*}
    \br{\frac{1}{c}}^T\fnot \leq \br{\frac{\gamma \lip\sqrt{d\log(n^2\log^2{(1/\beta')}/d\beta')}}{n\sqrt{\rho}}}^{\kappa} 
    \iff T \geq \frac{\bs{\log(\fnot) + \kappa\log\br{\frac{n\sqrt{\rho}}{\gamma \lip\sqrt{d\log(n^2\log^2{(1/\beta')}/d\beta')}}}}}{\log(c)}
\end{align*}}
Using the fact that $\log(c)=\log\Big(1+\fnot^{\frac{2-\kappa}{\kappa}}\frac{1}{32\gamma^2\wc)}\Big) \geq (1+32\fnot^{\frac{\kappa-2}{\kappa}}\gamma^2\wc)^{-1}$, the setting of $T = (1+32\fnot^{\frac{\kappa-2}{\kappa}}\gamma^2\wc)\Big[\log(\fnot) + \kappa\log\Big(\frac{n\sqrt{\rho}}{\gamma \lip \sqrt{d}}\Big)\Big]$ suffices to ensure convergence.
\end{proof}

\subsection{Lower bound for \texorpdfstring{ $\kappa\geq 2$}{}}

We give a lower bound on excess empirical risk for settings where the empirical risk satisfies $(1,\kappa)$-KL for $\kappa\geq 2$, under approximate differential privacy.

\begin{theorem}
\label{thm:lower-bound-kappa-geq-2}
   \sloppy
    Let $\kappa \geq 2, 0 <\epsilon \leq \ln{2}, 0 <\delta \leq \frac{1}{16}\br{1-e^{-\epsilon}}, d\in \bbN$ and  $B> 0$. For any $(\epsilon,\delta)$-DP procedure $\cA$, there exists a data space $\cX$, a set $\cW \subseteq \bbR^d$ containing a ball of radius $B$, a dataset $S$ and a convex loss function $w \mapsto f(w;x)$ which is $1$-Lipschitz over $\cW$, the 
    empirical loss $w\mapsto F(w;S)$ satisfies $(1,\kappa)$-KL, and 
    \begin{align*}
        \mathbb{E}_\cA\left[F(\cA(S);S) - \inf_{w\in \bbR^d} F(w;S)\right] \geq \frac{1}{4}\br{\frac{1}{n\epsilon}}^\kappa
    \end{align*}
\end{theorem}
\begin{proof}
    The proof adapts the construction of \cite{ALD21}, Theorem 5 from a lower bound on excess population risk under pure DP setting to that on excess empirical risk under approximate DP.
    We first prove a lower bound for $(1,\tau)$-growth functions, for $\tau \in (1,2]$.
    We recall the one-dimensional, unconstrained (so $\cW = \bbR^d$) construction in \cite{ALD21}, Theorem 5. The data space $\cX = \bc{-1,1}$, and for $a\in [0,1]$ to be specified later, define functions
    \begin{align*}
        f(w;1) = \begin{cases}
            \abs{w-a} & w\leq a \\
            \abs{w-a}^\tau & w> a
        \end{cases} \quad \text{and} \quad
        f(w;-1) = \begin{cases}
            \abs{w+a}^\tau & w\leq -a \\
            \abs{w+a} & w> -a
        \end{cases}
    \end{align*}
    The functions above are $1$-Lipschitz.
    Consider two datasets $S$ and $S'$ such that $S$ contains $\br{\frac{1+\rho}{2}}$ fraction of $1$'s and the rest $-1$'s. Similarly, $S'$ contains $\br{\frac{1-\rho}{2}}$ fraction of $1$'s and the rest $-1$'s. The number of points differing between $S$ and $S'$ is thus $n\rho$. We set $\rho = 1/n\epsilon$ to get $\frac{1}{\epsilon}$ differing points.
    The corresponding empirical risk functions are,
    \begin{align*}
        F(w;S) &= \br{\frac{1+\rho}{2}} f(w;1)+ \br{\frac{1-\rho}{2}} f(w;-1) \\
        F(w;S') &= \br{\frac{1-\rho}{2}} f(w;1)+ \br{\frac{1+\rho}{2}} f(w;-1)
    \end{align*}
    In the construction of \cite{ALD21}, Theorem 5, the above are their population risk functions ``$f_1(x)$'' and ``$f_{-1}(x)$''. Their minimizers are $w^*_S = a$ and $w^*_{S'}=-a$, with values $(1-\rho)a$ and $(1+\rho)a$ respectively. 
   Note that the above functions are convex. Further, with $a = \frac{\rho^{\frac{1}{(\tau-1)}}}{2} = \frac{1}{2(n\epsilon)^{\frac{1}{\tau-1}}}$, \cite{ALD21} showed that both functions $w\mapsto F(w;S)$ and $w\mapsto F(w;S')$ exhibit $(1,\tau)$-growth over all of $\bbR$. For any $(\epsilon,\delta)$-DP algorithm $\cA$, we have that,
    \begin{align*}
        &\sup_{\tilde S \in \bc{S,S'}} \mathbb{E}_\cA[F(\cA(\tilde S);\tilde S) - \inf_{w}F(w;\tilde S)] \\
        &\geq \frac{1}{2} \mathbb{E}_\cA\left[F(\cA(S);\tilde S) - F(w^*_S; S) + F(\cA(S');S') - F(w^*_{S'}; S')\right] \\
        & \geq \mathbb{E}_\cA\left[\abs{\cA(S)-w^*_S}^\tau +\abs{\cA(S')-w^*_{S'}}^\tau \right] \\
        & \geq \frac{1}{2}\br{\mathbb{E}_\cA\left[\abs{\cA(S)-w^*_S} +\abs{\cA(S')-w^*_{S'}} \right]}^\tau \\
        &\geq \frac{1}{4}\br{\mathbb{E}_\cA \left[ \abs{w_S^* - w_{S'}^*}\right]}^\tau \\
        & \geq \frac{1}{4}\br{\frac{1}{n\epsilon}}^{\frac{\tau}{\tau-1}}
    \end{align*}
where the second inequality uses the growth condition, the third uses that for $1\leq \tau\leq 2, |u+v|^\tau\leq 2 (|u|^\tau + |v|^\tau)$ and Jensen's inequality; the fourth uses Lemma 2 of \cite{chaudhuri2012convergence} and the final inequality plugs in computed distance between minimizers. Finally, the fact (Lemma \ref{lem:growth-implies-kl}) that convexity and $\br{1,\frac{\kappa}{\kappa-1}}$-growth implies $(1,\kappa)$-KL establishes the $\frac{1}{4}\br{\frac{1}{n\epsilon}}^{\kappa}$ lower bound for $(1,\kappa)$-KL functions. 
\end{proof}

%% file: sections/appendix-adaptive-gd.tex
\section{Missing Results from Section \ref{sec:adaptive-gd}} \label{app:adaptive-gd}

\subsection{Gradient Error of Algorithm \ref{alg:whp-ogd}} \label{app:adaptive-gd-grad-err}
\begin{lemma} \label{lem:adaptive-gd-grad-err}
Let $T+1$ denote the final value of $t$ reached during the run of Algorithm \ref{alg:whp-ogd}. With probability at least $1-2\beta$ under the randomness of Algorithm \ref{alg:whp-ogd}, for any $t\in[T]$ s.t. $\sigma_t = \frac{N_t}{\sqrt{d\log(n\sqrt{\rho}/\beta)}}$, it holds that 
\begin{align*}
\norm{\nabla_t-\nabla F(w_t;S)} 
&\leq N_t 
\leq \|\nabla F(w_t;S)\| + \frac{\lip\sqrt{\log(n\sqrt{\rho}/\beta)}}{\sqrt{n}\rho^{1/4}}.
\end{align*} 
Further, if for any $t\in[T]$, $\sigma_t = \frac{2\lip}{n\sqrt{\rho}}$ then $t=T$ and with probability at least $1-2\beta$ the above condition holds as well as $\norm{\nabla_T-\nabla F(w_T;S)} \leq \frac{\lip\sqrt{d\logterms}}{n\sqrt{\rho}}$.
\end{lemma}
\begin{proof}
Condition on the high probability event that for all $t\in [T]$, $\|\hat{b}_t\| \leq \sqrt{\log(n\sqrt{\rho}/\beta)}\hat{\sigma}_t = \frac{\lip\sqrt{\log(n\sqrt{\rho}/\beta)}}{\sqrt{n}\rho^{1/4}}$ and $\|b_t\| \leq \sqrt{d\log(n\sqrt{\rho}/\beta)}\sigma_t = N_t$.
This event happens with probability at least $1-2\beta$ due to the concentration properties of Gaussian noise and the fact that at most $n\sqrt{\rho}$ iterations are performed. Under this event we then have the following bound on the gradient error,
\begin{align*}
\norm{\nabla_t-\nabla F(w_t;S)} 
&\leq N_t 
\leq \|\nabla F(w_t;S)\| + \frac{\lip\sqrt{\log(n\sqrt{\rho}/\beta)}}{\sqrt{n}\rho^{1/4}}.
\end{align*}    
The second part of the lemma statement comes from the fact that when $\sigma_t = \frac{2\lip}{n\sqrt{\rho}}$, $\rho_t \geq \frac{\rho}{2}$ and the stopping condition is triggered. The second error bound result again comes from the concentration of Gaussian noise.
\end{proof}

\subsection{Proof of Theorem \ref{thm:adpative-gd-privacy} (Privacy of Algorithm \ref{alg:whp-ogd})} \label{app:adaptive-gd-privacy}
\begin{proof} 
Denote $T+1$ as the highest value attained by the variable $t$ during the run of the algorithm. Consider any round $t\in[T]$. We consider the privacy of the round conditional on $w_{t-1}$. Specifically, for the process of generating the gradient and gradient norm estimates at the $t$'th step, the scale of Gaussian noise ensures this process is $\rho_t$-zCDP. Specifically, %
\begin{align*}
    \rho_t = \br{\frac{\lip}{n\hat{\sigma}}}^2 + \br{\frac{\lip}{n\sigma_t}}^2 = \frac{\sqrt{\rho}}{n}+ \min\bc{\frac{\lip^2 d\logterms}{n^2 N_t^2},\frac{\rho}{2}}.  
\end{align*}
The $\frac{\rho}{2}$-zCDP guarantee of releasing the first $T-1$ iterates is then certified by the stopping condition (i.e. $\sum_{j=0}^t\rho_t \leq \frac{\rho}{2}$) and the fully adaptive composition properties of zCDP. That is, \cite[Theorem 1]{Whitehouse2022FullyAC} guarantees the privacy of the overall process even if the privacy bound at each iteration is chosen adaptively (rather than fixed a-priori as with standard composition theorems). 
Releasing the $T$'th iterate is also $\frac{\rho}{2}$-zCDP because $\sigma_T \geq \frac{2\lip}{n\sqrt{\rho}}$ and the sensitivity of any gradient estimate is at most $\frac{\lip}{n}$. %
Thus the overall algorithm is $\rho$-zCDP by composition.
\end{proof}

\subsection{Proof of Theorem \ref{thm:kl-ogd} (Convergence of Adaptive GD under KL* Condition)} \label{app:adaptive-gd-kl}
In the following we define $T+1$ as the highest attained value of $t$ during the run of Algorithm \ref{alg:whp-ogd} and define $c := 1+\frac{1}{8\gamma^2\smooth}$. 

Before proceeding with the main proof, it will be useful to first show that in the event that for some $t>0$ one has $\sigma_t = \frac{\lip}{n\sqrt{\rho}}$, the algorithm has reached its convergence criteria and stops.
\begin{lemma} \label{lem:adaptive-gd-small-noise-case}
Let $t>0$ and assume $\sigma_t = \frac{2\lip}{n\sqrt{\rho}}$. Then Algorithm \ref{alg:whp-ogd} stops at iteration $t$ and with probability at least $1-2\beta$ one has
\begin{align*}
    F(w_t;S) - F(w^*_\cS;S) = O\br{\br{\frac{\gamma\lip\sqrt{d\log(n\sqrt{\rho}/\beta)}}{n\sqrt{\rho}}}^\kappa + \br{\frac{\gamma^2\lip^2 \logterms}{n\sqrt{\rho}}}^{\kappa/2}}
\end{align*}
\end{lemma}
Importantly, this rate is strictly faster than the convergence claimed by Theorem \ref{thm:kappa-geq-2-result}. %
The proof is given in Appendix \ref{app:adaptive-gd-small-noise-case} and follows straightforwardly from the concentration of Gaussian noise and the KL condition.

Given this fact, we can proceed with the rest of the proof only considering the case where $\sigma_t = \frac{N_t}{\sqrt{d\log(n\sqrt{\rho}/\beta)}}$ for all $t\in\bc{0,...T}$.
We will first prove (under the stated assumption) the following useful lemma which roughly states that the excess risk is monotonically nonincreasing up to a certain threshold. Note in the following, we use $\Phi$, to denote \textit{exact} excess loss quantities. This in contrast to the analysis of Section \ref{sec:kappa-leq-2} where $\hat{\Phi}$ was used to indicate target excess risk loss thresholds. For the rest of this section, we assume $\frac{\lip^2 \logterms}{\smooth n}\leq \sqrt{\rho}$, as per the statement of Theorem \ref{thm:kl-ogd}.
\begin{lemma}\label{lem:decrease-to-threshold}
Define $\Phi_t = F(w_t;S) - F(w^*_\cS;S)$. Assume $F(\cdot;S)$ is $\smooth$-smooth.
Assume $\sigma_t = \frac{N_t}{\sqrt{d\log(n\sqrt{\rho}/\beta)}}$ for all $t\in[T]$.
Then with probability at least $1-2\beta$ we have for all $t\in[T]$ that
\begin{align}
    F(w_t;S) - F(w_{t+1};S) &\geq \frac{1}{8\smooth}\norm{\nabla F(w_t;S)}^2 - \frac{\lip^2 \log(n\sqrt{\rho}/\beta)}{\smooth n\sqrt{\rho}} \label{eq:whp-descent}
\end{align}
and if $F(\cdot;S)$ is also $(\gamma,\kappa)$-KL then
\begin{align*}
    \Phi_{t+1} \leq \max\bc{\Phi_t, 2\br{\frac{\max\bc{\gamma^2,1}\lip^2 \log(n\sqrt{\rho}/\beta)}{\smooth n\sqrt{\rho}}}^{\kappa/2}} 
\end{align*}
\end{lemma}
\begin{proof}
Throughout the following we condition on the high probabity event that
\begin{align*}
\norm{\nabla_t-\nabla F(w_t;S)} 
&\leq N_t 
\leq \|\nabla F(w_t;S)\| + \frac{\lip\sqrt{\log(n\sqrt{\rho}/\beta)}}{\sqrt{n}\rho^{1/4}}.
\end{align*}
which happens with probability at least $1-\beta$ by Lemma \ref{lem:adaptive-gd-grad-err} (given in Appendix \ref{app:adaptive-gd-grad-err}).
Now, standard descent lemma analysis yields
\begin{align}
    F(w_t;S) - F(w_{t+1};S) &\geq \ip{\nabla F(w_t;S)}{w_t - w_{t+1}} - \frac{\smooth}{2}\norm{w_{t+1}-w_t}^2 \nonumber \\
    &= \frac{1}{4\smooth}\norm{\nabla F(w_t;S)}^2 - \frac{1}{8\smooth}\norm{\nabla_t-\nabla F(w_t;S)}^2 \nonumber \\
    &\geq \frac{1}{8\smooth}\norm{\nabla F(w_t;S)}^2 - \frac{\lip^2 \logterms}{8 \smooth n\sqrt{\rho}}. \nonumber %
\end{align} 
This establishes the first claim of the lemma.

Continuing to the second claim, the above implies that if %
$\Phi_t \geq \br{\frac{\max\bc{\gamma^2,1} \lip^2 \logterms}{\smooth n\sqrt{\rho}}}^{\kappa/2}$,
we have by the KL condition that
\begin{align}
    \norm{\nabla F(w_t;S)}^2 &\geq \frac{1}{\gamma^{2}}\Phi_t^{2/\kappa} 
    \geq \frac{\lip^2 \logterms}{\smooth n\sqrt{\rho}}. \label{eq:grad-norm-lb}
\end{align} 
Thus we have
\begin{align}
    \Phi_t - \Phi_{t+1} = F(w_t;S) - F(w_{t+1};S) 
    &\geq \frac{1}{8\smooth}\norm{\nabla F(w_t;S)}^2 - \frac{\lip^2 \logterms}{8\smooth n\sqrt{\rho}} > 0 \label{eq:risk-decrease}
\end{align}
On the other hand, if 
$\Phi_t < \br{\frac{\gamma^2 \lip^2 \logterms}{\smooth n\sqrt{\rho}}}^{\kappa/2}$ 
then because using Eqn. \eqref{eq:risk-decrease} and the fact that $\norm{\nabla F(w_t;S)}\geq0$ we obtain %
\begin{align*}
    \Phi_{t+1} \leq \Phi_t + \frac{\lip^2 \logterms}{\smooth n\sqrt{\rho}} \leq 2\br{\frac{\max\bc{\gamma^2,1}\lip^2 \logterms}{\smooth n\sqrt{\rho}}}^{\kappa/2}.
\end{align*} 
Above we use the assumption that $\frac{\lip^2 \logterms}{\smooth n\sqrt{\rho}}\leq 1$.
Thus combining these two inequalities we have
$\Phi_{t+1} \leq \max\bc{\Phi_t,  2\br{\frac{\max\bc{\gamma^2,1}\lip^2 \logterms}{\smooth n\sqrt{\rho}}}^{\kappa/2}}$.
\end{proof}

The next lemma establishes how quickly the loss decreases. Specifically, we show that the loss decreases by a constant fraction after a certain number of steps. The smaller the excess risk is, the more steps are required to achieve this decrease. Recall $c := 1+\frac{1}{8\gamma^2\smooth}$. 
\begin{lemma} \label{lem:phase-decrease}
Let $K>0$ and $t\in[T]$ and assume the high probability event of Lemma \ref{lem:decrease-to-threshold} holds. Then for $K\geq (\frac{1}{c}\Phi_t)^{\frac{\kappa-2}{\kappa}} - 1$ it holds that 
$$\Phi_{t+K} \leq \max\bc{\frac{1}{c}\Phi_{t}, 2\br{\frac{\max\bc{\gamma^2,1}\lip^2 \logterms}{\min\{\smooth,1\} n\sqrt{\rho}}}^{\kappa/2}}.$$
\end{lemma}
\begin{proof}
We here condition on the high probability event that Lemma \ref{lem:decrease-to-threshold} holds (i.e. that the gradient error is bounded for the entire trajectory).
We proceed with a proof by contradiction. Assume by contradiction that
$$\Phi_{t+K} > \max\bc{\frac{1}{c}\Phi_{t}, 2\br{\frac{\max\bc{\gamma^2,1}\lip^2 \logterms}{\min\{\smooth,1\} n\sqrt{\rho}}}^{\kappa/2}}.$$
By Lemma \ref{lem:decrease-to-threshold}, this assumption implies the above inequality also holds for all $\Phi_{t+j}, j\in \bc{0,\dots,K}$,
{\small
\begin{align} \label{eq:phase-loss-lb}
    &\max\bc{\frac{1}{c}\Phi_{t}, 2\br{\frac{\max\bc{\gamma^2,1}\lip^2 \logterms}{\min\{\smooth,1\} n\sqrt{\rho}}}^{\kappa/2}} < \Phi_{t+K} \leq \max\bc{\Phi_{t+j}, 2\br{\frac{\max\bc{\gamma^2,1}\lip^2 \log(n\sqrt{\rho}/\beta)}{\smooth n\sqrt{\rho}}}^{\kappa/2}} \nonumber \\
    &\implies \max\bc{\frac{1}{c}\Phi_{t}, 2\br{\frac{\max\bc{\gamma^2,1}\lip^2 \logterms}{\min\{\smooth,1\} n\sqrt{\rho}}}^{\kappa/2}} \leq \Phi_{t+j}.
\end{align}}
The implication above uses the fact that 
$2\br{\frac{\max\bc{\gamma^2,1}\lip^2 \logterms}{\min\{\smooth,1\} n\sqrt{\rho}}}^{\kappa/2} \nless 2\br{\frac{\max\bc{\gamma^2,1}\lip^2 \logterms}{\smooth n\sqrt{\rho}}}^{\kappa/2}$.

We now sum over $K$ steps and using the descent lemma (see Lemma \ref{lem:decrease-to-threshold}, Eqn \eqref{eq:whp-descent}). We have
\begin{align}
    F(w_t;S) - F(w_{t+K};S) &\geq \sum_{j=1}^K \Big(\frac{1}{4\smooth}\norm{\nabla F(w_{t+j};S)}^2 - \frac{\lip^2 \logterms}{4\smooth n\sqrt{\rho}} \Big) \nonumber \\
    &\overset{(i)}{\geq}  \sum_{j=1}^K \Big(\frac{1}{4\gamma^2\smooth} \Phi_{t+j}^{2/\kappa} - \frac{\lip^2 \logterms}{4\smooth n\sqrt{\rho}} \Big)\nonumber \\
    &\overset{(ii)}{\geq}  \sum_{j=1}^K \frac{1}{2\gamma^2\smooth} \Phi_{t+j}^{2/\kappa}  \nonumber \\
    &\overset{(iii)}{\geq} \frac{(\frac{1}{c}\Phi_t)^{\frac{\kappa-2}{\kappa}}}{2\gamma^2\smooth} (\frac{1}{c}\Phi_{t})^{2/\kappa}
    = \frac{1}{2c\gamma^2\smooth} \Phi_{t} \label{eq:phase-loss-decrease}
\end{align}
Step $(i)$ uses the KL condition. Step $(ii)$ uses the fact that
 Eqn. \eqref{eq:phase-loss-lb} implies 
that for all $j\in \bc{0,\dots,K}$, $\Phi_{t+j} \geq \br{\frac{2 \gamma^2\lip^2 A}{n\sqrt{\rho}}}^{\kappa/2}$. 
The second inequality uses the KL condition. Step $(iii)$ uses the fact that $\Phi_{t+j} \geq \frac{1}{c}\Phi_t$, by Eqn. \eqref{eq:phase-loss-lb}, and the setting of $K$.
Manipulating Inequality \eqref{eq:phase-loss-decrease} above we have
\begin{align*}
    F(w_t;S) - F(w_{t+K};S) = \Phi_t - \Phi_{t+K} \geq \frac{1}{2c\gamma^2\smooth} \Phi_{t} \nonumber \\
    \implies \Phi_{t+K} \leq (1-\frac{1}{2c\gamma^2\smooth})\Phi_t = \frac{1}{c}\Phi_t. 
\end{align*}
This establishes the contradiction and thus 
$\Phi_{t+K} \leq \max\bc{\frac{1}{c}\Phi_{t}, 2\br{\frac{\max\bc{\gamma^2,1}\lip^2 \logterms}{\min\{\smooth,1\} n\sqrt{\rho}}}^{\kappa/2}}$.
\end{proof}

We can now prove Theorem \ref{thm:kl-ogd} itself. With the above two lemmas established, our primary concern is analyzing how the stopping conditions affect the convergence of the algorithm.
\begin{proof}[Proof of Theorem \ref{thm:kl-ogd}]
Condition on the high probability event that Lemma \ref{lem:decrease-to-threshold} holds (i.e. that the gradient error is bounded for the entire trajectory).
We will assume for the rest of the proof we assume that for all $t\in [T]$ that
\begin{equation} \label{eq:2nd-loss-lb-assumption} 
\Phi_t \geq 2\br{\frac{\max\bc{\gamma^2,1}\lip^2 \logterms}{\min\{\smooth,1\} n\sqrt{\rho}}}^{\kappa/2}
\end{equation}
Note that if for any $t\in[T]$ the above inequality is not satisfied, by Lemma \ref{lem:decrease-to-threshold} the convergence guarantee of Theorem \ref{thm:kl-ogd} is satisfied.

We now argue that the algorithm does not stop before convergence is reached by analyzing the stopping condition. 
It will be helpful to split the run of the algorithm into phases. We denote the first phase as the set of iterates $W_{1} = {w_0,w_1,...,w_{K_1}}$, where $K_1$ is the largest integer such that $F(w_{K_1};S) - F(w^*_\cS;S) \geq \frac{1}{c}\Phi_0$. Similarly define $W_{2} = {w_{K_1},w_{K_1+1},...,w_{K_2}}$ where $K_2$ is the largest integer such that $F(w_{K_2};S) - F(w^*_\cS;S) \geq \frac{1}{c}\Phi_{K_1}$, and so on for $W_3,W_4,...,W_p$. %
Our aim is to show the algorithm does not stop before convergence.

First, we bound the largest value $p$ can obtain without convergence. 
By Lemma \ref{lem:phase-decrease} and Eqn. \eqref{eq:2nd-loss-lb-assumption} we have $\Phi_{K_p} \leq \frac{1}{c^p}\fnot$. Thus for $p \geq p_{\mathsf{max}} := (1+8\gamma^2\smooth)\bs{\log(\fnot) + \frac{2\kappa}{4-\kappa}\log(n\sqrt{\rho}/[\gamma \lip])} \geq 
\frac{\log(\fnot) + \frac{2\kappa}{4-\kappa}\log(n\sqrt{\rho}/[\gamma \lip])}{\log(c)}$ %
we have
\begin{align*}
    \Phi_{K_p} \leq \br{\frac{1}{c}}^{p_{\mathsf{max}}}\fnot \leq \br{\frac{\gamma\lip }{n\sqrt{\rho}}}^{\frac{2\kappa}{4-\kappa}}
    < \br{\frac{c\gamma \lip\sqrt{p_{\mathsf{max}} d A}}{n\sqrt{\rho}}}^{\frac{2\kappa}{4-\kappa}}.
\end{align*}
Thus if the algorithm has not converged it must be the case that $p \leq p_{\mathsf{max}}$. Let us thus assume $p \leq p_{\mathsf{max}}$ for the following analysis.

The algorithm stops when $\sum_{t=0}^T \rho_t > \rho$. We observe (denoting $K_0 = 0$ for convenience)
\begin{align*}
    \sum_{t=0}^T \rho_t  
    &= \frac{T\sqrt{\rho}}{n}+ \frac{\lip^2d\logterms}{n^2}\sum_{j=1}^p \sum_{t=K_{j-1}}^{K_{j}} \frac{1}{N_t^2} \\
    &\leq \frac{T\sqrt{\rho}}{n}+ \frac{\lip^2d\logterms}{n^2}\sum_{j=1}^p \sum_{t=K_{j-1}}^{K_{j}} \br{\norm{\nabla F(w_t;S)}^2 - \frac{\lip^2 \logterms}{4\smooth n\sqrt{\rho}}}^{-1} \\
    &\leq \frac{T\sqrt{\rho}}{n}+ \frac{\lip^2d\logterms}{n^2}\sum_{j=1}^p \sum_{t=K_{j-1}}^{K_{j}} \frac{2}{\norm{\nabla F(w_t;S)}^2}
\end{align*}
The last step uses the fact that the KL condition and the loss lower bound assumed in Eqn. \eqref{eq:2nd-loss-lb-assumption} implies 
$\norm{\nabla F(w_t;S)}^2 \geq \frac{\lip^2 \logterms}{\smooth n\sqrt{\rho}}$.
Continuing, by the KL condition we have,
\begin{align*}
    \sum_{t=0}^T \rho_t %
    &\leq \frac{T\sqrt{\rho}}{n}+ \frac{\lip^2d\logterms}{n^2}\sum_{j=1}^p \sum_{t=K_{j-1}}^{K_{j}} \frac{2\gamma^2}{\Phi_{K_j}^{2/\kappa}} \\
    &\leq \frac{T\sqrt{\rho}}{n}+ \frac{\lip^2d\logterms}{n^2}\sum_{j=1}^p \Phi_{K_{j-1}}^{\frac{\kappa-2}{\kappa}}\frac{2\gamma^2}{\Phi_{K_j}^{2/\kappa}} \\
    &\leq \frac{T\sqrt{\rho}}{n}+ \frac{\lip^2d\logterms}{n^2}\sum_{j=1}^p \Phi_{K_{j}}^{\frac{\kappa-2}{\kappa}}\frac{2\gamma^2}{\Phi_{K_j}^{2/\kappa}} \\
    &\leq \frac{T\sqrt{\rho}}{n}+ \frac{2\gamma^2\lip^2d\logterms}{n^2} p_{\mathsf{max}} \max_{j\in[p]}\bc{\Phi_{K_j}^{\frac{\kappa-4}{\kappa}}}. 
\end{align*}

Thus, if $T \leq \frac{1}{2}n\sqrt{\rho}$, the algorithm has not stopped unless for some $t\in [T]$ we have 
\ifarxiv that \fi $\Phi_t = O\br{\br{\frac{\gamma \lip \sqrt{ d \logterms p_{\mathsf{max}}}}{n\sqrt{\rho}}}^{\frac{2\kappa}{4-\kappa}}}$.

To finish the proof, we consider the convergence when the algorithm stops after $T > \frac{1}{2}n\sqrt{\rho}$. 
Recall we are assuming the algorithm has run for at most $p \leq p_{\mathsf{max}}$ number of phases (as otherwise the algorithm has converged).
The number of iterations during each of each of these phases is at most 
$\Phi_{K_p}^{\frac{\kappa-2}{\kappa}}$. Thus the algorithm has not stopped unless 
\begin{align*}
    p_{\mathsf{max}} \Phi_{K_p}^{\frac{\kappa-2}{\kappa}} \geq \frac{1}{2}n\sqrt{\rho} 
    \implies \Phi_{K_p} \leq \br{\frac{2 p_{\mathsf{max}}}{n\sqrt{\rho}}}^{\frac{\kappa}{2-\kappa}}.
\end{align*}

To summarize, we now have three different bounds on the excess depending on three possible events. The first case is simply when $\Phi_T \leq 2\br{\frac{\max\bc{\gamma^2,1}\lip^2 \logterms}{\min\{\smooth,1\} n\sqrt{\rho}}}^{\kappa/2}$. The second case is when $\Phi_T \geq 2\br{\frac{\max\bc{\gamma^2,1}\lip^2 \logterms}{\min\{\smooth,1\} n\sqrt{\rho}}}^{\kappa/2}$ and $T \leq \frac{1}{2}n\sqrt{\rho}$, in which case we have shown $\Phi_T = O\br{\br{\frac{c\gamma \lip \sqrt{ d \logterms p_{\mathsf{max}}}}{n\sqrt{\rho}}}^{\frac{2\kappa}{4-\kappa}}}$. The final case is when $\Phi_T \geq 2\br{\frac{\max\bc{\gamma^2,1}\lip^2 \logterms}{\min\{\smooth,1\} n\sqrt{\rho}}}^{\kappa/2}$ and $T > \frac{1}{2}n\sqrt{\rho}$, in which case we have shown 
$\Phi_{T} \leq \br{\frac{p_{\mathsf{max}}}{n\sqrt{\rho}}}^{\frac{\kappa}{2-\kappa}}$. Combining these results yields the theorem statement.
\end{proof}

\subsection{Proof of Lemma \ref{lem:adaptive-gd-kl-star}} \label{app:adaptive-gd-trajectory}
\begin{proof}
We will prove the lemma result by induction. For any $t\in{0,...,T-1}$, assuming $w_{t}\in \cS$, we will show that $w_{t+1}\in\cS$. The base case for $w_0$ holds because $\cS$ is defined to contain $w_0$.

Before proceeding, we condition on the event that for all $t\in\bc{0,...,T-1}$ we have that
$\norm{\nabla_t-\nabla F(w_t;S)} 
\leq \|\nabla F(w_t;S)\| + \frac{\lip\sqrt{\log(n\sqrt{\rho}/\beta)}}{\sqrt{n}\rho^{1/4}}$, which happens with probability at least $1-2\beta$ by Lemma \ref{lem:adaptive-gd-grad-err} in Appendix \ref{app:adaptive-gd-grad-err}. %

To prove the induction step, let $w_{t}\in\cS$. 
We divide the proof into two cases, depending on $\|\nabla F(w_t;S)\|$. In the first case, assume $\|\nabla F(w_t;S)\| < \br{\frac{\lip^2\logterms}{n\sqrt{\rho}}}^{1/2}$. In this case, we will roughly prove that $w_{t+1}$ is in $\cS$ because it has not moved too far from $w_t$. Since $w_t \in \cS$, the KL condition holds at $w_t$. Thus the gradient norm bound and the KL condition imply
$F(w_t;S) - F(w^*_\cS;S) \leq \br{\frac{\gamma\lip^2\logterms}{n\sqrt{\rho}}}^{\kappa/2}$.
Let $R=2\br{\frac{\lip^2\logterms}{\smooth^2 n\sqrt{\rho}}}^{1/2}$ and recall we define the level set threshold as $\alpha= \max\big\{F(w_0;S) ,F(w^*_\cS;S)+\max\big\{F(w_0;S) ,F(w^*_\cS;S)+2(\gamma^{\kappa/2}+\lip)\br{\frac{\lip^2\logterms}{n\sqrt{\rho}}}^{1/2}\big\}\big\}$. For any point $w' \in \cB(w_t;R)$, by Lipschitzness one has 
\begin{align*}
    F(w';S) &\leq F(w_t;S) - F(w^*_\cS;s) + F(w^*_\cS;S) + \lip(\|\nabla F(w_t;S)\| + \|\nabla_t - \nabla F(w_t;S)\|) \\
    &\leq F(w^*_\cS;S) + 2\br{\frac{\gamma\lip^2\logterms}{n\sqrt{\rho}}}^{\kappa/2} + \lip\br{\br{\frac{\lip^2\logterms}{n\sqrt{\rho}}}^{1/2} + \frac{\lip\sqrt{\log(n\sqrt{\rho}/\beta)}}{\sqrt{n}\rho^{1/4}}} \\
    &\leq F(w^*_\cS;S) + 2(\gamma^{\kappa/2}+\lip)\br{\frac{\lip^2\logterms}{n\sqrt{\rho}}}^{1/2}
    \leq \alpha.
\end{align*}
Thus $\cB(w_t;R) \subseteq \cI$. Since $\cB(w_t;R)$ is path connected and $w_t\in\cS$, we have $\cB(w_t;R) \subseteq \cS$ by the definition of $\cS$.
Further, with probability at least $1-\beta$ we have %
\begin{align*}
    \|w_t - w_{t+1}\| &\leq \eta(\|\nabla F(w_t;S)\| + \|\nabla_t - \nabla F(w_t;S)\|) \\
    &\overset{(i)}{\leq} \eta\Big(2\|\nabla F(w_t;S)\| + \frac{\lip\sqrt{\log(n\sqrt{\rho}/\beta)}}{\sqrt{n}\rho^{1/4}} \Big) \overset{(ii)}{\leq} \frac{3}{2\smooth}\br{\frac{\lip^2\logterms}{n\sqrt{\rho}}}^{1/2} \leq R \\
    &\implies w_{t+1} \in \ball{w_t}{R}
\end{align*}
Above, step $(i)$ uses that the scale of noise in Algorithm \ref{alg:whp-ogd} guarantees with high probability that $\|\nabla_t - \nabla F(w_t;S)\|$. Step $(ii)$ uses the assumption that $\|\nabla F(w_t;S)\| \leq \br{\frac{\lip^2\logterms}{n\sqrt{\rho}}}^{1/2}$, the setting of $R$ (above) and $\eta = \frac{1}{2\smooth}$. As we have previously show, $\ball{w_t}{R} \subseteq \cS$, so we have shown $w_{t+1}\in\cS$.

We now consider the second case where $\|\nabla F(w_t;S)\| \geq \frac{\lip^2\logterms}{n\sqrt{\rho}}$. Consider the path parameterized by $l\in[0,1]$ defined by $\mathbf{w}(l) = w_{t} + l(w_{t+1}-w_{t})$. By the update rule of Algorithm \ref{alg:whp-ogd} and standard descent lemma analysis we have (using the smoothness of $F(\cdot;S)$)
\begin{align*}
    F(w_t;S) - F(\mathbf{w}(l);S) 
    &\geq \ip{\nabla F(w_t;S)}{w_t - \mathbf{w}(l)} - \frac{\smooth}{2}\norm{w_t-\mathbf{w}(l)}^2 \\
    &\geq l\ip{\nabla F(w_t;S)}{w_t - w_{t+1}} - \frac{l^2\smooth}{2}\norm{w_t-w_{t+1}}^2 \\
    &= \frac{l}{4\smooth}\norm{\nabla F(w_t;S)}^2 - \frac{l^2}{8\smooth}\norm{\nabla_t-\nabla F(w_t;S)}^2  \\
    &\geq \frac{l}{4\smooth}\norm{\nabla F(w_t;S)}^2 - \frac{l^2 \lip^2 \logterms}{8 \smooth n\sqrt{\rho}} \\
    &\overset{(i)}{\geq} l\bs{\frac{1}{4\smooth}\norm{\nabla F(w_t;S)}^2 - \frac{\lip^2 \logterms}{8 \smooth n\sqrt{\rho}}} \\
    &\overset{(ii)}{\geq} \frac{l}{4\smooth}\norm{\nabla F(w_t;S)}^2 \geq 0.
\end{align*} 
Step $(i)$ uses the fact that $l\leq 1$. Step $(ii)$ uses the assumption that $\|\nabla F(w_t;S)\| \geq \frac{\lip^2\logterms}{n\sqrt{\rho}}$. 
We have shown $F(w_t;S) \geq F(\mathbf{w}(l);S)$ for every $l\in[0,1]$. Thus $\bc{\mathbf{w}(l)}_{l\in[0,1]} \subseteq \cI$, and because $\bc{\mathbf{w}(l)}_{l\in[0,1]}$ is path connected and contains $w_t \in \cS$, we have $\bc{\mathbf{w}(l)}_{l\in[0,1]} \subseteq \cS$ and specifically $w_{t+1} \in \cS$.
\end{proof}

\subsection{Proof of Lemma \ref{lem:adaptive-gd-small-noise-case}} \label{app:adaptive-gd-small-noise-case}
\begin{proof}
First note that when $\sigma_t = \frac{\lip}{n\sqrt{\rho}}$ then $\rho_t > \rho$ and the algorithm stops. Furthermore, in this case we also have $N_t \leq \frac{\lip\sqrt{d\log(n\sqrt{\rho}/\beta)}}{n\sqrt{\rho}}$, and thus by the concentration of the noise we have with probability at least $1-\beta$ that
\begin{align*}
    \|\nabla F(w_t;S)\| \leq \frac{\lip\sqrt{d\log(n\sqrt{\rho}/\beta)}}{n\sqrt{\rho}} + \frac{\lip\sqrt{\logterms}}{\sqrt{n}\rho^{1/4}}
\end{align*}
The KL condition then implies that 
\begin{align*}
    F(w;S) - F(w^*_\cS;S) = O\br{\br{\frac{\gamma\lip\sqrt{d\log(n\sqrt{\rho}/\beta)}}{n\sqrt{\rho}}}^\kappa + \br{\frac{\gamma^2\lip^2 \logterms}{n\sqrt{\rho}}}^{\kappa/2}}
\end{align*}
\end{proof}

\subsection{Proof of Theorem \ref{thm:ogd} (Adaptive Gradient Descent without KL Condition)} \label{app:ogd}
In the following, we let $T+1$ denote the largest value attained by $t$ during the run of the algorithm. 

\paragraph{Privacy Proof}
The prove privacy we will use the following lemma.
\begin{lemma}\citep[Proposition 1.4]{Bun-zCDP} \label{lem:pure-dp-to-zcdp}
Any algorithm which is $(\epsilon,0)$-DP is also $(\frac{1}{2}\epsilon^2)$-zCDP.
\end{lemma}
The process of releasing $w_0,...,w_T$ is $\rho$-zCDP by Theorem \ref{thm:adpative-gd-privacy}. We thus only need to handel the additional privacy loss incurred via the use of the exponential mechanism to select $t^*$. Specifically, the exponential mechanism guarantees $(\sqrt{\rho},0)$-DP and is thus $\frac{1}{2}\rho$-zCDP by Lemma \ref{lem:pure-dp-to-zcdp}. The overall privacy is then $2\rho$-zCDP.

\paragraph{Convergence Proof}
Recall $t^*\in[T]$ is the index sampled by the exponential mechanism. Let $N^* = \min\limits_{t\in T}\bc{\norm{\nabla F(w_t;S)}}$. Note that the guarantees of the exponential mechanism (used to sample $t^*$) and scale of noise added to the gradient norm estimates we have with probability at least $1-2\beta$ that,
\begin{align}
    \|\nabla F(w_{t^*};S)\| &= \|\nabla F(w_{t^*};S)\| - N_{t^*} + N_{t^*} - N^* + N^* \nonumber \\
    &\leq \frac{\lip\logterms}{\sqrt{n}\rho^{1/4}} + \frac{4\lip\log(n\sqrt{\rho}/\beta)}{n\sqrt{\rho}} + N^*, \label{eq:n-star-bound}
\end{align}

We will now proceed to bound $N^*$.
First, if for any $t$ one has $\sigma_t = \frac{2\lip}{n\sqrt{\rho}}$, then $N_t \leq \frac{\lip\sqrt{d}}{n\sqrt{\rho}}$ and thus $N^* \leq \frac{\lip\sqrt{d}}{n\sqrt{\rho}}$. The convergence guarantees are then satisfied by Eqn. \eqref{eq:n-star-bound}.

We now turn towards the more difficult case where $\sigma_t=\frac{N_t}{\sqrt{d\log(n\sqrt{\rho}/\beta)}}$ for all $t\in[T]$. We start by analyzing the convergence of the algorithm in terms of the number of rounds $T$. 
By Lemma \ref{lem:decrease-to-threshold}, Eqn. \eqref{eq:whp-descent}, we have with probability at least $1-\beta$ that,
\begin{align*}
    F(w_t;S) - F(w_{t+1};S) 
    &\geq \frac{1}{8\smooth}\norm{\nabla F(w_t;S)}^2 - \frac{\lip^2 \logterms}{8 \smooth n\sqrt{\rho}}. 
\end{align*} 
Summing over all iterates and rearranging gives,
\begin{align}
    \frac{1}{T}\sum_{t=1}^T \norm{\nabla F(w_t;S)} \leq \sqrt{\frac{8 \fnot \smooth}{T}} + \frac{\lip\sqrt{\logterms}}{\sqrt{n}\rho^{1/4}}. \label{eq:ogd-T-convergence}
\end{align}

We now consider the worst case guarantee for Algorithm \ref{alg:whp-ogd}. 
Recall $N^* = \min\limits_{t\in T}\bc{\norm{\nabla F(w_t;S)}}$. We can use $N^*$ to lower bound the number of iterations made by the algorithm. We have,
\begin{align*}
    \sum_{t=0}^T \rho_t  
    &= \frac{T\sqrt{\rho}}{2n}+ \frac{\lip^2d\logterms}{n^2}\sum_{t=1}^T \frac{1}{N_t^2} \\
    &\leq \frac{T\sqrt{\rho}}{2n} + \frac{T \lip^2d\logterms}{n^2(N^*)^2}.
\end{align*}
Further by the stopping condition we have, 
\begin{align*}
T\br{\frac{\lip^2d\logterms}{n^2(N^*)^2} + \frac{\sqrt{\rho}}{2n}} \geq \frac{\rho}{2} \\
\implies T \geq \frac{1}{2}\min\bc{\frac{n^2(N^*)^2\rho}{\lip^2 d\logterms}, n\sqrt{\rho}}.
\end{align*}
By Eqn. \eqref{eq:ogd-T-convergence} we also have,
\begin{align*}
    N^* \leq \frac{1}{T}\sum_{t=1}^T \norm{\nabla F(w_t;S)} \leq \sqrt{\frac{8 \fnot \smooth}{T}} + \frac{\lip\logterms}{\sqrt{n}\rho^{1/4}}.
\end{align*}
Applying the above lower bound on $T$ to the upper bound on $N^*$ we obtain,
\begin{align*}
    N^* \leq \frac{3\lip\sqrt{\fnot\smooth d \logterms}}{nN^*\sqrt{\rho}} + \frac{\lip\logterms}{\sqrt{n}\rho^{1/4}}\\
    \implies N^* = \br{\frac{6\lip\sqrt{\fnot\smooth d \logterms}}{n\sqrt{\rho}}}^{1/2} + \frac{2\lip\logterms}{\sqrt{n}\rho^{1/4}}.
\end{align*}
Combining this bound with Eqn. \eqref{eq:n-star-bound} we have with probability at least $1-3\beta$ that,
\begin{align*}
    \|\nabla F(\bar{w};S)\| 
    &\leq \min\Bigg\{\sqrt{\frac{8 \fnot \smooth}{T}} + \frac{\lip\sqrt{\logterms}}{\sqrt{n}\rho^{1/4}}, \\ &+\br{\frac{6\lip\sqrt{\fnot\smooth d \logterms}}{n\sqrt{\rho}}}^{1/2} + \frac{3\lip\logterms}{\sqrt{n}\rho^{1/4}} + \frac{4\lip\logterms}{n\sqrt{\rho}}\Bigg\} \\
    &= O\br{\min\bc{\sqrt{\frac{\fnot \smooth}{T}},\br{\frac{\lip\sqrt{\fnot\smooth d}}{n\sqrt{\rho}}}^{1/2}} + \frac{\lip\sqrt{\logterms}}{\sqrt{n}\rho^{1/4}}}.
\end{align*}

%% file: sections/appendix-reg-lip-opt.tex
\section{Regularized Lipschitz Optimization} \label{app:dp-sc-whp}

In this section, we consider a function $\tilde f(w;x) = f(w;x) + \tilde L_1 \norm{w-w_0}^2$, where $w\mapsto f(w;x)$ is $L_0$-Lipschitz, $\tilde L_1$-weakly convex for all $x \in \cX$, and $w_0 \in \bbR^d$.
 It is well known that in such case, the function $w\mapsto \tilde f(w;x)$ is $\tilde L_1$ strongly convex (see, e.g. \citep{DD:2019,bassily2021differentially}). We denote the corresponding empirical risk as $\tilde F(w;S) = \frac{1}{n}\sum_{i=1}^n f(w;x_i) + \tilde L_1 \norm{w-w_0}^2$.

The following result is a rate of $\tilde O\br{\frac{L_0^2d}{\tilde L_1n^2\rho}}$ on excess empirical risk via Noisy Gradient Descent, Algorithm \ref{alg:noisy-gd}. Multiple works have investigated closely related settings \citep{feldman2020private,asi2021private}, but due to our specific requirements (i.e. unconstrained setting and only assuming convexity of the regularized loss function) we provide a more tailored result here. %

\begin{algorithm}[h]
\caption{Noisy Gradient Descent}
\label{alg:noisy-gd}
\begin{algorithmic}[1]
\REQUIRE Dataset $S$, 
zCDP paramter $\rho$,
initial point $w_0\in\re^d$, 
probability $\beta$, Lipschitz parameter $\lip$, Weak convexity $\wc$, step size sequence $\bc{\eta_t}_t$, number of iterations $T$, noise standard deviation $\sigma$.

\FOR{$t=1 \ldots T-1$} 
\STATE $\xi_t \sim \cN(0, \sigma^2\bbI)$
\STATE $w_{t+1} = \Pi_{\cB_{\frac{L_0}{2\tilde L_1}}\br{w_0}}\br{w_t - \eta_t \br{\nabla \tilde F(w;S) + \xi_t}}$
\ENDFOR
\STATE Return $\bar w = \frac{2}{T(T+1)}\sum_{t=1}^Ttw_t$
\end{algorithmic}
\end{algorithm}

\begin{theorem} \label{thm:dp-sc-whp} Let  $\rho>0$.
 Algorithm \ref{alg:noisy-gd} with 
  $T = \frac{n^2\rho \log^2{(2/\beta)}}{d}$, 
 $\eta_t = \frac{1}{\tilde L_1 t}$ and 
 $\sigma^2 = \frac{4L_0^2T}{n^2\rho}$
 satisfies 
 $\rho$-zCDP.
 Further, with probability at least $1-\beta$, the excess empirical risk of its output, $\bar w$, is bounded as,
    \label{thm:sc-ngd}
     \begin{align}
        \tilde F(\bar w;S) - \tilde F(w^*;S) = O\br{\frac{L_0^2d\log{(n^2\log^2{(2/\beta)}/d\beta)}}{\tilde L_1n^2\rho}}
    \end{align}
\end{theorem}
\begin{proof}
    The privacy proof is based on the observation that, even though the function $w \mapsto \tilde f(w;x)$ may not be Lispchitz, the sensitivity of the gradient, in every iteration, is controlled, since it is a sum of a Lipschitz and (data-independent) regularizer. In particular, the sensitivity of gradient at every iteration is bounded by $\frac{2L_0}{n}$. With the stated noise variance, applying the guarantee of Gaussian mechanism for zCDP and composition \citep{Bun-zCDP}, completes the privacy analysis.

    The utility proof is based on standard high-probability convergence analysis of (S)GD for strongly convex optimization \citep{harvey2019simple}.
    We first show that the unconstrained minimizer, $w^*$, lies in the constrained set $\cB_{\frac{L_0}{2\tilde L_1}}\br{w_0}$. From the optimality criterion for unconstrained convex optimization, we have that $2\tilde L_1\norm{w^* - w_0} = \norm{\nabla F(w^*;S)} \leq L_0$. This implies that $\norm{w^*-w_0} \leq \frac{L_0}{2\tilde L_1}$. This also gives us that the function $w\mapsto \tilde F(w;S)$ is $2L_0$-Lipschitz over the constrained set.

    From Gaussian concentration \citep{jin2019short}, we have that, with probability, at least $1-\beta/2$, for all $t \in [T]$, $\norm{\xi_t} \leq \sqrt{d\log{(2T/\beta)}}\sigma$.  Further, conditioned on the above, we have from Lemma 4 in \cite{harvey2019simple}, that with probability at least $1-\beta$, 
    \begin{align*}
        \sum_{t=1}^T\ip{\xi_t}{w_t-w^*} = O\br{\frac{L_0}{\tilde L_1}\sqrt{d\log{(2T/\beta)}}\sigma \log{(2/\beta)}T}. 
    \end{align*}

    The rest of the analysis is repeating the proof of Theorem 3.1 in \cite{harvey2019simple}. We get,
    \begin{align*}
        \tilde F(\bar w;S) - \tilde F(w^*;S) &= O\br{\frac{L_0^2}{\tilde L_1T} + \frac{\sigma^2d\log{(2T/\beta)}}{\tilde L_1T}} + \frac{1}{T(T+1)}  \sum_{t=1}^T\ip{\xi_t}{w_t-w^*} \\
        & = O\br{\frac{L_0^2}{\tilde L_1T} + \frac{\sigma^2d\log(2T/\beta)}{\tilde L_1 T} + \frac{L_0\sqrt{d\log{(2T/\beta)}}\sigma \log{(2/\beta)}}{\tilde L_1 T}} \\
        &= O\br{\frac{L_0^2d\log{(n^2\log^2{(2/\beta)}/d\beta)}}{\tilde L_1 n^2\rho}}
    \end{align*}
    where the last step follows by setting of $\sigma$ and $T$.
    \end{proof}